\documentclass[twoside]{article}
\usepackage[accepted]{aistats2026}

\usepackage{cancel}
\usepackage{hyperref}
\usepackage{url}
\usepackage{natbib}
\usepackage{amsmath,amssymb,amsfonts, amsthm}
\newtheorem{theorem}{Theorem}
\newtheorem{assumption}{Assumption}
\newtheorem{lemma}{Lemma}
\newtheorem{corollary}{Corollary}
\usepackage{algorithm}
\usepackage{algpseudocode}
\usepackage{xcolor}

\newcommand{\Unif}{\mathrm{Unif}}
\newcommand{\Categorical}{\mathrm{Categorical}}
\newcommand{\ones}{\mathbf{1}}
\makeatletter
\renewcommand{\section}{\@startsection{section}{1}{\z@}
{-2.0ex plus -0.5ex minus -.2ex}
{1.5ex plus 0.3ex minus .2ex}
{\large\bfseries\raggedright\MakeUppercase}}
\makeatother

\usepackage{enumitem}
\begin{document}

\runningauthor{Srikanth, Gaur, Aggarwal}

\twocolumn[
\aistatstitle{Discrete State Diffusion Models: A Sample Complexity Perspective}

\aistatsauthor{ Aadithya Srikanth$^{*}$ \And Mudit Gaur$^{*}$ \And Vaneet Aggarwal }

\aistatsaddress{ 
   Purdue University, West Lafayette, IN \\ 
   \texttt{\{srikan18, mgaur, vaneet\}@purdue.edu} \\
   \textsuperscript{*}Equal contribution
}

]

\begin{abstract}
Diffusion models have demonstrated remarkable performance in generating high-di\-men\-sion\-al samples across domains such as vision, language, and the sciences. Although con\-tin\-u\-ous-state diffusion models have been extensively studied both empirically and theoretically, dis\-crete-state diffusion models, essential for applications involving text, sequences, and combinatorial structures, remain significantly less understood from a theoretical standpoint. In particular, all existing analyses of dis\-crete-state models assume score estimation error bounds without studying sample complexity results. In this work, we present a principled theoretical framework for dis\-crete-state diffusion, providing the first sample complexity bound of $\widetilde{\mathcal{O}}(\epsilon^{-2})$. Our structured decomposition of the score estimation error into statistical, ap\-prox\-i\-ma\-tion, optimization, and clipping components offers critical insights into how dis\-crete-state models can be trained efficiently. This analysis addresses a fundamental gap in the literature and establishes the theoretical tractability and practical relevance of discrete-state diffusion models.
\end{abstract}

\section{Introduction} \label{sec: Introduction}
Diffusion models \citep{song2021scorebased} have gained significant attention lately due to their empirical success in various generative modeling tasks. They have been used in computer vision and audio generation \citep{ulhaq2022efficient,bansal2023universal},  text generation \citep{li2022diffusion}, sequential data modeling \citep{tashiro2021csdi}, policy learning  \citep{chen2024deep}. They have also shown promise in life sciences \citep{jing2022torsional,malusare2024improving,malusare2025augmenting}, and biomedical image reconstruction \citep{chung2022scorebaseddiffusionmodelsaccelerated}.

They model an unknown and unstructured distribution through forward and reverse stochastic processes. In the forward process, samples from the dataset are gradually corrupted to obtain a stationary distribution. In the reverse process, a well-defined noisy distribution is used as the initialization to iteratively produce samples that resemble the learnt distribution using the \emph{score} function. Most of the earlier works have been on continuous-state diffusion models, where the data is defined on the Euclidean space $\mathcal{R}^d$, unlike the discrete data, which is defined on $\mathcal{S}^d$, where $\mathcal{S}$ is the discrete set of values each component can take and $d$ is the number of components or dimension of the data. 

Generative AI applications involving discrete data have grown rapidly in recent years. These include text generation \citep{zhou2023diffusionnatselfpromptingdiscretediffusion} and summarization \citep{dat2025discretediffusionlanguagemodel}. Other examples include combinatorial optimization \citep{li2024fast}, molecule generation and drug discovery \citep{malusare2024improving}, protein and DNA sequence design \citet{Alamdari2023.09.11.556673, avdeyev2023dirichletdiffusionscoremodel}, and graph generation \citet{qin2024sparsetrainingdiscretediffusion}, there has been an increased interest in diffusion with discrete data. For a broad review of related literature, see \citet{ren2025fastsolversdiscretediffusion}. 

\citet{chen2024convergenceanalysisdiscretediffusion} initiated the theoretical study of score-based discrete-state diffusion models. They proposed a sampling algorithm for the hypercube $\{0,1\}^d$, using uniformization under the Continuous Time Markov Chain (CTMC) framework to establish convergence guarantees and bounds on iteration complexity. \citet{zhang2025convergencescorebaseddiscretediffusion} further extended this theory to a general state space, $[S]^d$, and proposed a discrete-time sampling algorithm for high-dimensional data. They provided KL convergence bounds centered on truncation, discretization, and score estimation error decomposition. \citet{ren2025discretecontinuousdiffusionmeet} proposed a L\'evy-type stochastic integral framework for the error analysis of discrete-state diffusion models. They implemented both $\tau$ leaping and uniformization algorithms and presented error bounds in KL divergence, the first for the $\tau$ leaping scheme. However, it is important to note that these error analyses were performed under the assumption of access to an \(\epsilon_{score}\)-accurate score estimator.  Prior works by \citet{chen2024convergenceanalysisdiscretediffusion}, \citet{zhang2025convergencescorebaseddiscretediffusion}, and \citet{ren2025discretecontinuousdiffusionmeet} laid the foundation for the theoretical study and analyzed iteration complexity, while the sample complexity is an unexplored topic. In this work, we bridge this gap by providing the first sample complexity bounds for discrete-state diffusion models, offering rigorous insight into the number of samples required to estimate the score function to within a desired accuracy. We perform a detailed analysis of the score estimation error bound $\epsilon_{score}$, which was simply assumed in earlier works. 

Specifically, we seek an answer to the following question.
\begin{center}
\textit{How many samples are required for a sufficiently expressive neural network to estimate the discrete-state score function well enough to generate high-quality samples such that the KL divergence between the data distribution and the generated marginal at the final step is at most \(\epsilon\)?}
\end{center}
We leverage the strong convexity of the negative entropy function over the closed set on which the true and approximate score functions are defined to upper bound the Bregman divergence of the score estimation error by the squared Euclidean norm. We further decompose this into \emph{approximation, statistical, optimization, and clipping errors}. Approximation error arises from the limited capacity of the chosen function class to represent the target function. Statistical error results from having only a finite dataset, which leads to estimation inaccuracies. Optimization error reflects the inability to reach the global minimum during training. Clipping error is due to violating the constraints on the network output. 
We thus provide a practical context that involves limitations in neural network estimation, a limited number of dataset samples, and a finite number of stochastic gradient descent (SGD) steps. 

To the best of our knowledge, this work offers the first theoretical investigation into the sample complexity of discrete-state diffusion models. We introduce a recursive analytical framework that traces the optimization error incurred at each step of SGD, allowing us to quantify the effect of performing only a limited number of optimization steps when learning the discrete-state score. 
Our approach bounds the discrepancy between the empirical and expected loss directly over a finite hypothesis class, thereby bypassing complexity measures that scale poorly with network size. An essential component of our analysis is the application of the Polyak–Łojasiewicz (PL) condition (Assumption \ref{ass:PL-condition}), which ensures a quadratic lower bound on the error and is crucial for regulating the overall error in score estimation.

We summarize our main contributions as follows:
\begin{itemize}
     \item \textbf{First sample complexity bounds for discrete-state diffusion models} To the best of our knowledge, this is the first work to present rigorous sample complexity analyses under reasonable and practical assumptions for discrete-state diffusion models. We derive an order optimal bound of $\widetilde{\mathcal{O}}(\epsilon^{-2})$,  and thus providing deeper insight into the sample efficiency of discrete-state diffusion. In particular we prove that $\widetilde{\mathcal{O}}(\epsilon^{-2})$ many samples are enough for discrete-state diffusion models to generate samples from a distribution that is $\epsilon$-close to
the target distribution in KL divergence.

    \item \textbf{Principled error decomposition} We propose a structured decomposition of the score estimation error into approximation, statistical, optimization, and clipping components, enabling the characterization of how each factor contributes to the overall sample complexity. 
\end{itemize}

\section{Related Work}\label{sec: Related Works}
\subsection{Discrete-state Diffusion Models}
Discrete-state diffusion models were initially studied as discrete-time Markov chains, which train and sample the model at discrete time points. \citet{NEURIPS2021_958c5305} first proposed Discrete Denoising Diffusion Probabilistic Models (D3PMs) as the discrete-state analog of the ``mean predicting''
 DDPMs \citep{ho2020denoising}. These models were largely applied to fields other than language due to empirical challenges \citep{lou2024discretediffusionmodelingestimating}. Although these models offer various practical sampling algorithms, they lack flexibility and theoretical guarantees \citep{zhang2025convergencescorebaseddiscretediffusion}. \citet{campbell2022continuoustimeframeworkdiscrete} first introduced the continuous-time (CTMC) framework, similar to the NCSNv3 by \citet{song2021scorebased}. The time-reversal of a forward CTMC is fully determined by its rate matrix and the forward marginal probability ratios, which together define the discrete-state score function. 

\subsection{Convergence and Theoretical Analysis of Discrete-state Diffusion Models}
While convergence analysis for continuous diffusion models is well-established, theoretical understanding in the discrete setting remains comparatively underdeveloped. \citet{campbell2022continuoustimeframeworkdiscrete} provided one of the earliest error bounds for discrete-state diffusion by analyzing a $\tau$-leaping sampling algorithm under total variation distance, assuming bounded forward probability ratios and $L^\infty$ control on the rate matrix approximation. However, the bound exhibits at least quadratic growth in the dimension $d$. In contrast, \citet{chen2024convergenceanalysisdiscretediffusion} introduced a score-based sampling method for discrete-state diffusion models by simulating the reverse process via CTMC uniformization on the binary hypercube, where forward transitions involve independent bit flips. Their method yields convergence guarantees and algorithmic complexity bounds that scale nearly linearly with $d$, comparable to the best-known results in the continuous setting \citep{benton2024nearly}. \citet{zhang2025convergencescorebaseddiscretediffusion} and \citet{ren2025discretecontinuousdiffusionmeet} perform KL divergence error analysis by decomposing into truncation, discretization, and approximation error, and provide iteration complexities of sampling algorithms. However, these analyses assume the score estimation error as simply bounded by a constant ( denoted by $\epsilon_{\mathrm{score}}$), effectively bypassing the score-learning problem. 

\indent Several recent works have analyzed the sample complexity of score-based continuous diffusion models. For example, \citet{guptaimproved} presents a bound of $\widetilde{\mathcal{O}}(\epsilon^{-5})$ using a quantile-based reformulation and assuming access to the ERM. In addition, \citet{gaur2025improvedsamplecomplexitydiffusion} derives a sample complexity bound of $\widetilde{\mathcal{O}}(\epsilon^{-4})$ by relaxing the assumption of access to ERM. Extending these results to discrete-state diffusion models, we establish for the first time a sample complexity bound of $\widetilde{\mathcal{O}}(\epsilon^{-2})$.

\paragraph{Notation} 
Lowercase letters are used to denote scalars, and boldface lowercase letters to represent
vectors. The $i$-th entry of a vector $\mathbf{x}$ is denoted by $x^i$. 
We use $x^{\backslash i}$ to refer to all dimensions of $x$ except the $i$-th, 
and $x^{\backslash i} \odot \hat{x}^i$ to denote a vector whose $i$-th dimension takes the value $\hat{x}^i$, while the other dimensions remain as $\mathbf{x}^{\backslash i}$. 
For a positive integer $n$, we denote $[n]$ as the set $\{1,2,\ldots,n\}$, 
$1_n \in \mathbb{R}^n$ as the vector of ones, and $I_n \in \mathbb{R}^{n \times n}$ as the identity matrix. The notation \( y \lesssim z \) means that there exists a universal constant \( C' > 0 \) such that \( y \leq C' z \). We use $
x \;\asymp\; y$ to indicate $x$ is asymptotically on the order of $y$. The notation $e_{i}$ refers to a one-hot vector with a 1 in the $i$-th position. $\delta_{x,y}$ denotes the Kronecker delta, which equals $1$ if $x=y$ and $0$ otherwise.
\section{Preliminaries and Problem Formulation} \label{sec: Preliminaries and Problem Formulation}
We begin this section by formally introducing the fundamentals of discrete-state diffusion through the lens of Continuous-time Markov Chain (CTMC). 
\subsection{Discrete-state diffusion processes}

The probability distributions are considered over a finite support $\mathcal{X} = [S]^d$, where each sequence $x \in \mathcal{X}$ has $d$ components drawn from a discrete set of size $S$. The forward marginal distribution $q_t$ can be represented as a probability mass vector $q_t \in \Delta^N$, where $\Delta^N$ is the probability simplex in $\mathbb{R}^N$, and $N = S^d$ is the cardinality of the discrete-state space. The forward dynamics of this distribution are governed by a continuous-time Markov process described by the Kolmogorov forward equation \citep{campbell2022continuoustimeframeworkdiscrete}.
\begin{equation} 
\frac{dq_t}{dt} = Q_t^{\top} q_t,\; 
q_0 \approx q_{\text{data}}\; 
\label{eqn:forward_process}
\end{equation}
Here, $Q_t \in \mathbb{R}^{N \times N}$ is a time-dependent generator (or rate) matrix with non-negative off-diagonal entries and rows summing to zero. This ensures conservation of total probability mass as the distribution evolves over time. We assume a 
finite uniform departure rate as required by the uniformization algorithm, which bounds the state-dependent exit rates by a single global rate  \citep{zhang2025convergencescorebaseddiscretediffusion, chen2024convergenceanalysisdiscretediffusion}. 
\begin{equation}
    \lambda = \sup_{x \in \mathcal{X}} \sum_{y \neq x} Q_t(x,y) < \infty
\label{eqn:lambda}
\end{equation}
To approximate this continuous evolution in practice, a small-step Euler discretization is applied. The resulting transition probability for a step of size $\Delta t$ is given by \citep{zhang2025convergencescorebaseddiscretediffusion, lou2024discretediffusionmodelingestimating}.
\begin{equation}
q(x_{t+\Delta t} = y \mid x_t = x) = \delta{\{x,y\}} + Q_t(x, y) \Delta t + o(\Delta t)
\label{eqn: time_disc}
\end{equation}

Equation (\ref{eqn: time_disc}) captures the probability of transitioning from state $x$ to $y$, $Q_t(x,y)$ governing the instantaneous transition rate and $q(x_{t+\Delta t} = y \mid x_t = x)$ is the infinitesimal transition probability of being in state $y$ at time $t + \Delta t$ given state $x$ at time $t$. The process also admits a well-defined reverse-time evolution \citep{kelly2011reversibility}, with $q_{T-t}$ the time-reversed marginal
and is characterized by another diffusion matrix, $\overline{Q}_t$, as follows:
\begin{align}
\frac{d q_{T-t}}{dt} &= \overline{Q}_{T-t}^T q_{T-t},\; 
\overline{Q}_t(x,y) = \tfrac{q_t(y)}{q_t(x)} Q_{t}(y,x)\; 
\label{eqn: reverse_time_score}
\end{align}
The ratio $\left(\frac{q_t(y)}{q_t(x)}\right)_{y \neq x} \in \mathbb{R}^{(N-1)}$ in Equation (\ref{eqn: reverse_time_score}) is known as the \emph{concrete score} $s_t(x)$ \citep{meng2023concretescorematchinggeneralized} and is analogous to $\nabla_x$log$(p_t)$ in continuous-state diffusion models \citep{song2020generativemodelingestimatinggradients}. A neural network-based score estimator $\hat{s}_{\theta,t}(\cdot)$ of $s_t(\cdot)$ is learned by minimizing the score entropy loss \citep{lou2024discretediffusionmodelingestimating, zhang2025convergencescorebaseddiscretediffusion}for $t \in [0,T]$ as given by:
\begin{equation}
\begin{aligned}
&\mathcal{L}_{\text{SE}}(\hat{s}_{\theta}) \\
&= \int_{0}^{T} \mathbb{E}_{x_t} 
\sum_{y \neq x_t} Q_t(y,x_t) \, D_I\!\left(s_t(x_t)_y \,\|\, \hat{s}_{\theta,t}(x_t)_y\right)\, dt
\end{aligned}
\label{eqn: score_entropy}
\end{equation}
where the expectation is taken over ${x_t \sim q_t}$ and is characterized by the Bregman divergence $D_{I}(\cdot)$ unlike the $L^2$ loss widely adopted in the continuous case. The Bregman divergence \( D_{I}(\cdot) \), with respect to the negative entropy function \( I(\cdot) \), also known as the generalized I-divergence, is defined as:
\begin{equation}
D_I(\mathbf{x} \| \mathbf{y}) = \sum_{i=1}^d \left[ -x^i + y^i + x^i \log \frac{x^i}{y^i} \right]
\label{eqn:bregman}
\end{equation}
In Equation (\ref{eqn:bregman}), $I(\cdot)$ refers to the negative entropy function, where $I(x) = \Sigma_{i=1}^{d}x^i\log x^i$ for a $d$ dimensional data. It is important to note that the Bregman divergence does not satisfy the triangle inequality.  
We defer more details on the 
Bregman divergence and convexity of $I(\cdot)$ to Lemma \ref{lemma: bregman_decomposition} in Appendix \ref{sec:supp_lemm}. 

\subsection{Forward process}
The forward evolution $X = (X_t)_{t \geq 0}$ is formulated as a CTMC over the discrete state space $[S]^d$. It is initialized by drawing $X_0$ from the data distribution $p_{\text{data}}$. The marginal law at time $t$, denoted by $q_t = \text{Law}(X_t)$, evolves according to the Kolmogorov forward equation (Equation (\ref{eqn:forward_process})).\\  
Since directly working with a rate matrix $Q_t$ of size $S^d \times S^d$ is infeasible in high dimension, the model assumes a factorized structure \citep{zhang2025convergencescorebaseddiscretediffusion, chen2024convergenceanalysisdiscretediffusion}. In particular, the conditional distribution of the process can be written as $q_{t|s}(x_t \mid x_s) = \prod_{i=1}^d q^i_{t|s}(x_t^i \mid x_s^i)$, so that each coordinate evolves independently as its own CTMC, with forward rate $Q^{\text{tok}}_t$. Under this construction, the global rate matrix $Q_t$ is sparse: transitions between states are only possible when their Hamming distance is one as adopted by \citet{campbell2022continuoustimeframeworkdiscrete}. A common choice for the per-coordinate dynamics is the time-homogeneous uniform flip rate, given by
\begin{equation}
Q^{\text{tok}} = \frac{1}{S} \, 1_S 1_S^\top - I_S
\end{equation}
Combining these coordinate-wise generators produces the global time-homogeneous rate matrix $Q$, whose entries are given by \citep{zhang2025convergencescorebaseddiscretediffusion}
\begin{equation}
Q(x,y) = 
\begin{cases}
\frac{1}{S}, & \text{if } \mathrm{Ham}(x,y) = 1, \\
\left(\tfrac{1}{S} - 1 \right)d, & \text{if } x=y, \\
0, & \text{otherwise}.
\end{cases}
\end{equation}
The chain evolves by allowing each coordinate to flip independently, but at any given instant only one coordinate flip occurs. Thus transitions always correspond to Hamming distance one. We defer the details on the factorisation and effective score space to Appendix \ref{app: score_dim}. As $t \to \infty$, the marginal $q_t$ converges to the uniform distribution $\pi^d$ over the entire state space $[S]^d$ \citep{chen2024convergenceanalysisdiscretediffusion, zhang2025convergencescorebaseddiscretediffusion}.

\subsection{Reverse process}
Formally, the true reverse process $(U_t)_{t \in [0,T]}$ begins at $U_0 \sim q_T$ and obeys 
$\mathrm{Law}(U_t) = q_{T-t}$. Its generator is denoted $Q^{\leftarrow}_t$ and is defined so that 
the reverse marginal $q_{T-t}$ satisfies the reverse Kolmogorov equation as given in Equation (\ref{eqn: reverse_time_score}).
In analogy with the forward dynamics, the reverse generator $Q_t^{\leftarrow}(x, \tilde{x})$ only has nonzero entries when the states $x$ and $\tilde{x}$ differ in exactly one coordinate. More precisely, it takes the form \citep{zhang2025convergencescorebaseddiscretediffusion}
\begin{equation}
Q_t^{\leftarrow}(x, \tilde{x}) \;=\;
\sum_{i=1}^d Q^{\text{tok}}(\tilde{x}^i, x^i)\,
\delta\{x^{\backslash i}, \tilde{x}^{\backslash i}\}\,
\frac{q_{T-t}(\tilde{x})}{q_{T-t}(x)}
\end{equation}
Here, $Q^{\text{tok}}$ is the time-homogenous per-coordinate rate matrix of the forward chain, the indicator $\delta\{x^{\backslash i}, \tilde{x}^{\backslash i}\}$ enforces that $x$ and $\tilde{x}$ differ only in coordinate $i$. This construction shows that the reverse process also evolves by flipping one coordinate at a time, but the rates are adjusted by the relative likelihood of the target state compared to the current state under the reverse marginal $q_{T-t}$. This ratio defines the \emph{discrete-state score function}, which is central to training the reverse-time model. From the structure of the rate matrix as defined above, Equation (\ref{eqn: score_entropy}) can be compactly written as follows:
\begin{equation}
\mathcal{L}_{\text{SE}}(\hat{s}_{\theta}) 
= \frac{1}{S} \int_0^T \mathbb{E}_{\mathbf{x}_t \sim q_t} 
D_I\!\left(s_t(\mathbf{x}_t) \,\|\, \hat{s}_{\theta, t}(\mathbf{x}_t)\right) dt
\label{eqn: score_entropy_bd}
\end{equation}

In particular, \citet{lou2024discretediffusionmodelingestimating} further showed that the score entropy loss is equivalent to minimizing the path measure KL divergence between the true and approximate reverse processes.

\subsection{Problem Formulation}

Let the discrete-state score function be approximated by a parameterized family of functions 
\(\mathcal{F}_\Theta = \{\hat{s}_\theta : \theta \in \Theta\}\), where each 
\(\hat{s}_\theta : [S]^d \times [0,T] \to \mathbb{R}^{d(S-1)}\) 
is typically represented by a neural network with depth \(D\) and width \(W\). Where \(D\) indicates the number of layers in the neural network and \(W\) indicates the maximum number of neurons in any layer of the neural network.
Given i.i.d.\ samples \(\{x_i\}_{i=1}^n \sim p_{\text{data}}\) from the data distribution, 
the score network is trained by minimizing the following time-indexed loss:
\begin{equation}
\mathcal{L}_{\text{SE}}(\hat{s}_{\theta}) 
= \frac{1}{S} \int_0^T \mathbb{E}_{\mathbf{x}_t \sim q_t} 
D_I\!\left(s_t(\mathbf{x}_t) \,\|\, \hat{s}_{\theta, t}(\mathbf{x}_t)\right) dt
\end{equation}
where \(q_t = \mathrm{Law}(X_t)\) is the forward CTMC marginal at time \(t\), 
\(s_t\) is the true discrete-state score function defined by
\begin{equation}
s_t(x)_{i,\hat{x}^i} 
= \frac{q_t(x^{\backslash i}\odot \hat{x}^i)}{q_t(x)},\;
x \in [S]^d,\; i \in [d],\; \hat{x}^i \neq x^i
\end{equation}
and \(D_I(\cdot\|\cdot)\) is the Bregman divergence induced by negative entropy.  

\paragraph{Reverse-time marginal process.}
Let $V = (V_t)_{t \in [0,T]}$ denote the learned reverse process constructed from 
score estimators $\hat{s}_\theta(\cdot)$, initialized from the noise distribution $\pi^d$. 
Its law at time $t$ is denoted by
\[
p_t = \mathrm{Law}(V_t).
\]

\paragraph{Objective.}
Our goal is to quantify how well the learned reverse-time process approximates the 
true data distribution \(p_{\text{data}}\) in terms of the KL divergence. 
Specifically, we aim to show the number of samples needed, so that with high probability, 
\begin{equation}
D_{\mathrm{KL}}(p_{\text{data}} \,\|\, p_T) \;\leq\; \widetilde{\mathcal{O}}(\epsilon).
\end{equation}

\section{Assumptions}
To formally establish our sample complexity results, we begin by outlining the core assumptions that underlie our analysis. We then define both the population and empirical loss functions used in training score-based discrete-state diffusion models. This is followed by a decomposition of the total error into statistical, approximation, optimization, and clipping errors, which are individually analyzed in the upcoming sections.
Let us define the population loss at time $k$ for $k \in [0,K-1]$ as
\begin{equation}
\begin{aligned}
\mathcal{L}_k(\theta) 
&= \int_{kh}^{(k+1)h}\mathbb{E}_{\mathbf{x}_t \sim q_t} \left[ 
D_I \left( 
s_k(\mathbf{x}_t) 
\,\|\, 
\hat{s}_{\theta, k}(\mathbf{x}_t) 
\right) 
\right]dt
\end{aligned}
\label{eqn:PRM_loss_assumption_2}
\end{equation}
The corresponding empirical loss is:
\begin{align}
\widehat{\mathcal{L}}_k(\theta) = 
\frac{1}{n_k} \sum_{i=1}^{n_k} 
D_I\left(
s_t(\mathbf{x}_{k,t,i}) 
\,\|\, 
\hat{s}_{\theta, k}(\mathbf{x}_{k,t,i})
\right),\nonumber\\
\quad t \sim \text{Unif}[kh,(k+1)h]
\label{eqn:ERM_loss_assumption_1}
\end{align}
where $n_k$ is the sample size at each time step. $x_{k,t,i}$ denotes the forward-diffused sample obtained at a random time $t$ within the $k$-th discretization interval, where $i$ indexes the training example from a dataset of size $n_k$. Here, $k$ denotes the discretization step corresponding to the interval $[kh,(k+1)h]$, and $t \sim \mathrm{Unif}[kh,(k+1)h]$ is a uniformly sampled continuous time.
\begin{assumption}[Polyak - \L{}ojasiewicz (PL) condition.]
\label{ass:PL-condition}
The loss $\mathcal{L}_k(\theta)$ for all $k \in [0,K-1]$ satisfies the Polyak--\L{}ojasiewicz condition, i.e.,  there exists a constant $\mu_{k} > 0$ such that
\begin{align}
    \frac{1}{2} \| \nabla \mathcal{L}_k(\theta) \|^2 \geq \mu_{k} \left( \mathcal{L}_k(\theta) - \mathcal{L}_k(\theta^*) \right), \quad \forall~ \theta \in \Theta
\end{align}
where $\theta^* = \arg \min_{\theta \in \Theta} \mathcal{L}_k(\theta)$ denotes the global minimizer of the population loss.
\end{assumption}

The Polyak-Łojasiewicz (PL) condition is significantly weaker than strong convexity and is known to hold in many non-convex settings, including overparameterized neural networks \citep{liu2022loss}. Prior works in continuous-state diffusion models such as \citet{guptaimproved} and \citet{block2020generative} implicitly assume access to an exact empirical risk minimizer (ERM) for score function estimation, as reflected in their sample complexity analyses (see Assumption A2 in \citet{guptaimproved} and the definition of $\hat{f}$ in Theorem 13 of \citet{block2020generative}). This assumption, however, introduces a major limitation for practical implementations, where the exact ERM is not attainable. In contrast, the PL condition allows us to derive sample complexity bounds under realistic optimization dynamics, without requiring exact ERM solutions. 


\begin{assumption}[Smoothness and bounded gradient variance of the score loss.]
\label{ass:smoothness}
For all $k\in [0, K-1]$, the population loss $\mathcal{L}_{k}(\theta)$ is $\kappa$-smooth with respect to the parameters $\theta$, i.e., for all $\theta, \theta' \in \Theta$ 
\begin{align}
\| \nabla \mathcal{L}_k(\theta) - \nabla \mathcal{L}_k(\theta') \| \leq \kappa \| \theta - \theta' \|
\end{align}
We assume that the estimators of the gradients ${\nabla}\mathcal{L}_k(\theta)$ have bounded variance.
\begin{align}
\mathbb{E}\| \nabla \mathcal{\widehat{L}}_k(\theta) - \nabla \mathcal{L}_k(\theta) \|^2 \leq \sigma^{2}
\end{align}
\end{assumption}

The bounded gradient variance assumption is a standard assumption used in SGD convergence results such as \citet{koloskova2022sharper} and \citet{ajalloeian2020convergence}.
This assumption holds for a broad class of neural networks under mild conditions, such as Lipschitz-continuous activations (e.g., GELU), bounded inputs, and standard initialization schemes \citep{allen2019convergence, du2019gradient}.
The smoothness condition ensures that the gradients of the loss function do not change abruptly, which is crucial for stable updates during optimization. This stability facilitates controlled optimization dynamics when using first-order methods like SGD or Adam, ensuring that the learning process proceeds in a well-behaved manner. 

\begin{assumption}[Bounded initial score]
\label{full_support}
The data distribution \(p_{\text{data}}\) has full support on \(\mathcal{X}\), and there exists a constant \(B >  0\), such that for all \(\mathbf{x}\in\mathcal{X}\), \(i\in[d]\), and \(x^{i}\neq \hat{x}^{i}\in[S]\),
\begin{equation}
s_{0}(\mathbf{x})_{i,\hat{x}^{i}}
= \frac{q_{0}\!\bigl(\mathbf{x}^{\backslash i}\!\odot\! \hat{x}^{i}\bigr)}{q_{0}(\mathbf{x})}
= \frac{p_{\text{data}}\!\bigl(\mathbf{x}^{\backslash i}\!\odot\! \hat{x}^{i}\bigr)}{p_{\text{data}}(\mathbf{x})}
\;\in\; \left(\dfrac{1}{B},\, B\right)
\end{equation}

\end{assumption}
We assume the data distribution 
$p_{\text{data}}$ has full support and satisfies a uniform score bound 
independent of the dimension $d$ \citep{zhang2025convergencescorebaseddiscretediffusion,  chen2024convergenceanalysisdiscretediffusion}. Consequently, our results in Theorem \ref{thm:main_theorem} do not require early stopping, and we set $\delta = 0$. Even though \cite{zhang2025convergencescorebaseddiscretediffusion} do not explicitly assume a lower bound $1/B$, they require it for their Proposition 3 to hold.  This assumption 
ensures well-posedness of the score function at initialization and 
enables dimension-free control of the score magnitude. Additionally, $\kappa_i = \tfrac{(p^i_{\text{data}})_{\max}}{(p^i_{\text{data}})_{\min}}$ is the ratio of the largest to smallest entry in the marginal distribution of the $i$-th dimension, measuring how unbalanced that marginal is and $\kappa^2 = \sum_{i=1}^d \kappa_i^2$ aggregates this imbalance across all dimensions.  \citet{lou2024discretediffusionmodelingestimating} and \citet{zhang2025convergencescorebaseddiscretediffusion} show that minimizing the discrete score entropy loss Equations (\ref{eqn: score_entropy}) and (\ref{eqn: score_entropy_bd}) is equivalent to minimizing the KL divergence between the path measures of the true reverse and sampling diffusion processes, i.e., the distributions over their entire trajectories.

\section{Theoretical Results}

We now present our main theoretical results, establishing finite-sample complexity bounds and KL convergence analyses for the discrete-state diffusion model under the stated assumptions.

\begin{theorem}[Sample complexity and KL guarantee]
\label{thm:main_theorem}
Suppose Assumptions \ref{ass:PL-condition}, \ref{ass:smoothness}, and \ref{full_support} hold. With a probability at least $1-
\lambda$ the KL divergence between $p_{data}$ and $p_{T}$ is bounded by
\begin{align}
&D_{KL}(p_{data} \,\|\, p_{T}) 
\lesssim\; de^{-T}\log S 
+ C \kappa^2 S^2 h^2 T 
+ \tfrac{M\lambda Th}{S} \nonumber\\
&+ \frac{C^3}{S} \sum_{k=0}^{K-1} h \,
  \mathcal{O}\!\left( (W)^L\!\left((S-1)d+\tfrac{L}{W}\right)
  \sqrt{\tfrac{\log\!\left(\tfrac{2K}{\gamma}\right)}{n_{k}}}\;\right)
\label{eqn:main_bound}
\end{align}

where $M=C(S-1)d$, provided that the score function estimator $\hat{s}_{\theta,t}(\cdot)$ is parameterized by a neural network with sufficient expressivity such that its width $W \geq (S-1)d$, and the number of training samples per step $n_k$ satisfies
\[
n_k \;=\; \tilde{\Omega}\!\left(
\frac{C^{6}}{S^{2}\epsilon^{2}}\;
W^{\,2L}\!\left((S-1)d+\frac{L}{W}\right)^{2}
\right)
\]

In this case, with probability at least $1-\gamma$, the distribution $p_{T}$ obtained from simulating the reverse CTMC up to time $T$ matches the target distribution $p_{data}$ up to error $\tilde{O}(\epsilon)$.
\end{theorem}

\begin{corollary}
Suppose Assumptions \ref{ass:PL-condition}, \ref{ass:smoothness}, and \ref{full_support} hold. For any $\epsilon > 0$, if we set
\[
T \;\asymp\; \log\!\left(\frac{d \log S}{\epsilon}\right), 
h \;\asymp\; \min\!\left\{
\Big(\tfrac{\epsilon}{C \kappa^2 S^2 T}\Big)^{1/2},
\;\tfrac{\epsilon S}{M\lambda T}
\right\}
\]
then the discrete diffusion model requires
\[ 
\max\!\left\{
\sqrt{\tfrac{C \kappa^2 S^2}{\epsilon}}\,
\big[\log(\tfrac{d\log S}{\epsilon})\big]^{3/2},
\;\tfrac{M\lambda}{\epsilon S}\,
\big[\log(\tfrac{d\log S}{\epsilon})\big]^2
\right\}
\]
steps to reach a distribution $p_{T}$ such that
\[
D_{\mathrm{KL}}\!\left(p_{data} \,\|\, p_{T}\right) \;\lesssim\; \tilde{O}(\epsilon)
\]
provided the following order optimal sample complexity bound is satisfied
\[
n_k \;=\; \tilde{\Omega}\!\left(
\frac{C^{6}}{S^{2}\epsilon^{2}}\;
W^{\,2L}\!\left((S-1)d+\frac{L}{W}\right)^{2}
\right)
\]
\end{corollary}


The first term in the KL divergence bound in (\ref{eqn:main_bound}) captures the truncation error due to the insufficient mixing of the forward process. The second term reflects the discretization error due to approximating the continuous-time reverse process with a fixed number of discrete steps of size $h$. The third term is a discretization error arising from replacing the continuous-time score-error integral by its discrete-time approximation. The fourth term quantifies the approximation, statistical, and optimization errors arising from finite network capacity, finite training data, and imperfect minimization of the empirical loss.

For some distributions with large KL separation, it implies from the standard mean–estimation lower bound that $n=\Omega(\varepsilon^{-2})$ samples are needed to reach $\varepsilon$-KL accuracy. We defer the details on the hardness of learning to Appendix \ref{hardness}. This result provides a rigorous sample complexity guarantee for learning score-based discrete-state diffusion models under realistic training assumptions.  


\section{Proof of Theorem \ref{thm:main_theorem}}\label{main_proof}

Let the ideal reverse process be 
\(U = (U_t)_{t \in [0,T]}\), initialized at \(U_0 \sim q_T\) with path measure \(\mathbb{Q}\). 
The sampling process is 
\(V = (V_t)_{t \in [0,T]}\), initialized at \(V_0 \sim \pi^d\) with path measure \(\mathbb{P}^{\pi^d}\), 
and the auxiliary process is 
\(\widetilde{V} = (\widetilde{V}_t)_{t \in [0,T]}\), which shares the dynamics of \(V\) but starts from \(q_T\), with path measure \(\mathbb{P}^{q_T}\).
At time \(t\), we denote \(q_t = \mathrm{Law}(U_{T-t})\) (true reverse-time marginal) and 
\(p_t = \mathrm{Law}(V_t)\) (sampling marginal). 
By the data processing inequality, the KL divergence between the final-time marginals of the true reverse and sampling processes is always less than the KL divergence over the entire path measure, and can be written as:
\begin{equation}
D_{\mathrm{KL}}\!\left(q_0 \,\|\, p_{T}\right) 
\;\leq\; D_{\mathrm{KL}}\!\left(\mathbb{Q}\,\|\,\mathbb{P}^{\pi^d}\right)
\label{eq:kl-marginal-vs-path}
\end{equation}

and by the chain rule of KL divergence,
\begin{equation}
D_{\mathrm{KL}}\!\left(\mathbb{Q}\,\|\,\mathbb{P}^{\pi^d}\right)
= D_{\mathrm{KL}}\!\left(q_T \,\|\, \pi^d\right) 
+ D_{\mathrm{KL}}\!\left(\mathbb{Q}\,\|\,\mathbb{P}^{q_T}\right)
\label{eq:kl-chain-rule}
\end{equation}
$D_{\mathrm{KL}}(q_T \,\|\, \pi^d)$ captures the mismatch in initial distributions (truncation error), and the second term $D_{\mathrm{KL}}(\mathbb{Q} \,\|\, \mathbb{P}^{q_T})$ accounts for the discrepancy between the true and approximate reverse trajectories, including both discretization and approximation errors. By combining (\ref{eq:kl-marginal-vs-path}) and (\ref{eq:kl-chain-rule}), we obtain:
\begin{align}
&D_{\mathrm{KL}}(q_0 \,\|\, p_{T})
\leq 
D_{\mathrm{KL}}(\mathbb{Q} \,\|\, \mathbb{P}^{\pi^d}) \nonumber\\
&= D_{\mathrm{KL}}(q_T \,\|\, \pi^d) + D_{\mathrm{KL}}(\mathbb{Q} \,\|\, \mathbb{P}^{q_T})
\label{eqn:KL_inequality}
\end{align}
Therefore, to bound \( D_{\mathrm{KL}}(q_0 \,\|\, p_{T}) \), it is sufficient to bound the two KL divergence terms on the right-hand side. \citet{zhang2025convergencescorebaseddiscretediffusion} have shown in their Lemma 1 through a Girsanov-based method that $D_{\mathrm{KL}}(\mathbb{Q} \,\|\, \mathbb{P}^{q_T})$ is equal to the Bregman divergence between the true reverse process and the discretized reverse sampling process (see Lemma \ref{lem:kl-divergence_path_measure}). Specifically,
\begin{align}
&D_{\mathrm{KL}}(\mathbb{Q} \,\|\, \mathbb{P}^{q_T}) \nonumber\\ 
&= \frac{1}{S} \sum_{k=0}^{K-1} \int_{kh}^{(k+1)h} 
\mathbb{E}_{x_t}\!\left( D_I\!\left(s_t(\mathbf{x}_t) \,\|\, \hat{s}_{\theta,(k+1)h}(\mathbf{x}_t)\right) \right) dt
\label{eqn: kl_decomp_continuous}
\end{align}
where the expectation is taken over ${\mathbf{x}_t \sim q_t}$ and \( D_I\left(s_t(\mathbf{x}_t) \,\|\, \hat{s}_{\theta, (k+1)h}(\mathbf{x}_t)\right) \) is the Bregman divergence characterizing the distance between the true score at continuous time \( s_t(\cdot) \) and the approximate score at discrete time points \( \hat{s}_{\theta,(k+1)h}(\cdot)\), for \( t \in [kh, (k+1)h] \). Here, \( h \) denotes the discretization step size, and \( k \) is the number of discrete time steps. \citet{zhang2025convergencescorebaseddiscretediffusion} further decomposes Equation~(\ref{eqn: kl_decomp_continuous}) into a score-movement term (discretization error) and a score-error term (approximation error), as given in (\ref{eqn:kl_decomp_discrete}), through the strong convexity of \( I(\cdot) \). This decomposition facilitates analysis under their assumption of a bounded score estimation error \( \epsilon_{\mathrm{score}} \). In contrast, we avoid this assumption and instead perform a rigorous analysis of the score-error term in Theorem \ref{thm:main_theorem} to determine the number of samples required to ensure that the overall KL divergence \( D_{\mathrm{KL}}(q_0 \,\|\, p_{T}) \leq \tilde{O}(\epsilon) \), as stated in Equation~(\ref{eqn:KL_inequality}). Using the Bregman divergence property established in Lemma \ref{lemma: bregman_decomposition} $(i)$, following decompostion of $D_{\mathrm{KL}}(\mathbb{Q}\,\|\,\mathbb{P}^{q_T})$ is obtained:
\begin{align}
&D_{\mathrm{KL}}(\mathbb{Q}\,\|\,\mathbb{P}^{q_T}) \lesssim\nonumber\\
& \tfrac{C^2}{S} \sum_{k=0}^{K-1} \int_{kh}^{(k+1)h} 
\mathbb{E}\, D_I\!\left(s_{(k+1)h}(\mathbf{x}_t)\,\|\,\hat{s}_{\theta,(k+1)h}(\mathbf{x}_t)\right) dt \nonumber\\
& + \tfrac{C}{S} \sum_{k=0}^{K-1} \int_{kh}^{(k+1)h} 
\mathbb{E}\, \| s_t(\mathbf{x}_t) - s_{(k+1)h}(\mathbf{x}_t) \|_2^2 \, dt \label{eqn:kl_decomp_discrete}
\end{align}
where the expectations are taken over ${\mathbf{x}_t \sim q_t}$. Earlier works such as  \citet{zhang2025convergencescorebaseddiscretediffusion}  directly bound the first term on the right-hand side as $C^2\epsilon_{score}$. Specifically, their paper assumes access to an $\epsilon_{score}$-accurate score estimator. 
The notation \( y \lesssim z \) means that there exists a universal constant \( C' > 0 \) such that \( y \leq C' z \).
However, further leveraging the strong convexity of \(I(\cdot)\), along with Lemma \ref{lemma: bregman_decomposition} $(ii)$, we further upper bound the path measure KL divergence in (\ref{eqn:kl_decomp_discrete}) as follows:
\begin{align}
&D_{\mathrm{KL}}(\mathbb{Q}\,\|\,\mathbb{P}^{q_T})\lesssim\; \nonumber\\
&\ \frac{C}{S} \sum_{k=0}^{K-1} 
\int_{kh}^{(k+1)h} 
\mathbb{E}_{x_t} 
\bigl\| s_t(\mathbf{x}_t) - s_{(k+1)h} (\mathbf{x}_t) \bigr\|_2^2 \, dt + \nonumber \\
& \frac{C^3}{S} \sum_{k=0}^{K-1} 
\int_{kh}^{(k+1)h} 
\mathbb{E}_{\mathbf{x}_t} 
\bigl\| s_{(k+1)h}(\mathbf{x}_t) - \hat{s}_{\theta,(k+1)h}(\mathbf{x}_t) \bigr\|_2^2 \, dt
\label{eqn:kl_decomp_discrete_l2}
\end{align}
where the expectations are taken over $\mathbf{x}_t \sim q_t$. 
Note that in order to satisfy the conditions required for Lemma  \ref{lemma: bregman_decomposition}, the score functions $s_{t}(x_{t})$, $s_{(k+1)h}(x_{t})$, and   $\hat{s}_{\theta, (k+1)h}(x_{t})$ must be bounded.  While the true score functions are bounded through Lemma \ref{lemma: score_uniform_bound}, the estimated score function is bounded component-wise to the range $[1/C,C]$ through hard-clipping of the final layer \citep{zhang2025convergencescorebaseddiscretediffusion}. From the data processing inequality (\ref{eqn:KL_inequality}), we obtain the following bound on $D_{\mathrm{KL}}(p_{data} \,\|\, p_{T})$ and is explained in detail in Appendix \ref{app: appendix_main_theorem_full_support}:
\begin{align}
&D_{KL}(p_{data} \,\|\, p_{T}) \;=\ D_{KL}(q_{0} \,\|\, p_{T}) \nonumber\\
&\;\lesssim\; de^{-T}\log S + C \kappa^2 S^2 h^2 T  + \frac{M\lambda T h}{S} + \frac{C^3}{S} \sum_{k=0}^{K-1} h A_k
\label{eq:final-kl-bound_main_text}
\end{align}

The last term in (\ref{eq:final-kl-bound_main_text}) characterizes the score approximation error and $M = C(S-1)d$. To bound this error in terms of the number of samples $n_k$ required for training, we now consider the following loss per discrete time step. Here, for notational simplicity, we refer to the discrete time step $k$ obtained by freezing the continuous time $t$ at the upper limit of each interval. Specifically, we write, $k: = (k+1)h$. To facilitate our analysis, we introduce the unbounded ($\overline{s}_{\theta,k}(\cdot)$) and the clipped ($\hat{s}_{\theta, k}(\mathbf{x}_k)$) score functions. During the reverse sampling process, we use the clipped score function $\hat{s}_{\theta, k}(\mathbf{x}_k)$ by projecting $\overline{s}_{\theta,k}(\cdot)$ into the interval $[1/C, C]$ component-wise. Using these definitions, we obtain the following loss per discrete time step $k$:
\begin{align}
A_k
&= \mathbb{E}_{\mathbf{x}_k \sim q_k}
\bigl\| s_k(\mathbf{x}_k) - \hat{s}_{\theta,k}(\mathbf{x}_k) \bigr\|_2^2 \nonumber \\
&\le
\underbrace{2\,\mathbb{E}_{\mathbf{x}_k \sim q_k}
\bigl\| s_k(\mathbf{x}_k) - \bar{s}_{\theta,k}(\mathbf{x}_k) \bigr\|_2^2}_{B_k}
\quad + \nonumber \\ &\quad 
\underbrace{2\,\mathbb{E}_{\mathbf{x}_k \sim q_k}
\bigl\| \bar{s}_{\theta,k}(\mathbf{x}_k) - \hat{s}_{\theta,k}(\mathbf{x}_k) \bigr\|_2^2}_{C_k}
\label{eq:Ak_decomp}
\end{align}
where the inequality in \eqref{eq:Ak_decomp} follows from $(a+b)^2 \le 2(a^2+b^2)$ by writing
\begin{align*}
A_k
&=
\mathbb{E}_{\mathbf{x}_k \sim q_k}
\Bigl\|
\bigl( s_k(\mathbf{x}_k) - \bar{s}_{\theta,k}(\mathbf{x}_k) \bigr) \\
&\qquad\qquad
+
\bigl( \bar{s}_{\theta,k}(\mathbf{x}_k) - \hat{s}_{\theta,k}(\mathbf{x}_k) \bigr)
\Bigr\|_2^2.
\end{align*}

The term $B_{k}$ represents the difference between the true score function and the unbounded score function $\overline{s}_{\theta,k}(\cdot)$ obtained during training. The term $C_{k}$ accounts for the error incurred due to violating the constraints of the network. We show in Lemma \ref{lemma:clip_error} that the clipping error $C_{k} \le B_{k}$. This implies that
\begin{equation}
A_k \leq 4B_k
\label{eqn: a_k_b_k_in} 
\end{equation}
Now in order to upper bound $B_{k}$, we decompose it into three different components: approximation error, statistical error, and optimization error. In order to facilitate the analysis of the three terms listed above, we define the following neural network parameters.
\begin{equation}
\theta_k^a = \arg\min_{\theta \in \Theta} \; \int_{kh}^{(k+1)h}\mathbb{E}_{x_t} \left[ 
D_I \left( 
s_k(\mathbf{x}_t) 
\,\|\, 
\bar{s}_{\theta, k}(\mathbf{x}_t) 
\right) 
\right]dt
\label{eqn:PRM_compact}
\end{equation}
\begin{equation}
\begin{aligned}
\theta_k^b 
&= \arg\min_{\theta \in \Theta} \; 
\frac{1}{n} \sum_{i=1}^{n} 
D_I\!\left( s_t(\mathbf{x}_{k,t,i}) \,\|\, \overline{s}_{\theta,k}(\mathbf{x}_{k,t,i}) \right) \\
&\quad t \sim \text{Unif}[kh,(k+1)h]
\end{aligned}
\label{eqn:ERM_compact}
\end{equation}

Moreover, the term
\begin{align}
B_k \;\leq\;&\;
\underbrace{4\,\mathbb{E}_{\mathbf{x}_k \sim q_k} 
\bigl\| s^{a}_{\theta,k}(\mathbf{x}_k) - s_k(\mathbf{x}_k) \bigr\|_2^2}_{\mathcal{E}_k^{\text{approx}}} \nonumber \\
&+ \underbrace{4\,\mathbb{E}_{\mathbf{x}_k \sim q_k} 
\bigl\| s^{a}_{\theta,k}(\mathbf{x}_k) - s^{b}_{\theta,k}(\mathbf{x}_k) \bigr\|_2^2}_{\mathcal{E}_k^{\text{stat}}} \nonumber \\
&+ \underbrace{4\,\mathbb{E}_{\mathbf{x}_k \sim q_k} 
\bigl\| \overline{s}_{\theta,k}(\mathbf{x}_k) - s^{b}_{\theta,k}(\mathbf{x}_k) \bigr\|_2^2}_{\mathcal{E}_k^{\text{opt}}}
\label{eqn:Ak_decomp}
\end{align}

where, $s^a_{\theta,k}(\cdot)$ and $s^b_{\theta,k}(\cdot)$ are the score functions associated with the parameters $\theta^a_k$ and $\theta^b_k$. Here, the loss function given in Equation (\ref{eqn:PRM_compact}) is the expected value of the loss function minimized at training time to obtain the unbounded estimate of the score function $\overline{s}_{\theta}(\cdot)$. The loss function in Equation (\ref{eqn:ERM_compact}) is the empirical loss function that actually is minimized to obtain $\overline{s}_{\theta}(\cdot)$ at training time. 

The term \textit{Approximation error} $\mathcal{E}_k^{\mathrm{approx}}$ captures the error due to the limited expressiveness of the function class $\{\hat{s}_\theta\}_{\theta \in \Theta}$. The \textit{statistical error} $\mathcal{E}_k^{\mathrm{stat}}$ is the error from using a finite sample size. The \textit{optimization error} $\mathcal{E}_k^{\mathrm{opt}}$ is due to not reaching the global minimum during training.



Drawing upon recent advances in statistical learning theory and neural network approximation bounds, we bound each of the above errors through Lemmas \ref{lemma:approx_error}, \ref{lemma:stat_error}, and \ref{lemma:opt_error} to bound the score-estimation error in terms of the number of samples and neural network parameters. The detailed proofs are provided in Appendix \ref{appendix_section_inter_lemmas} and \ref{app: appendix_main_theorem_full_support}. To prove Theorem \ref{thm:main_theorem}, we bound each of the errors in (\ref{eqn:Ak_decomp}) and subsequently use (\ref{eqn:kl_decomp_discrete_l2}) to ensure \( D_{\mathrm{KL}}(p_{data} \,\|\, p_{T }) \leq \tilde{O}(\epsilon) \). We introduce the following lemmas to bound the approximation, statistical, and optimization errors. The formal proofs corresponding to each of the lemmas are provided in Appendix \ref{appendix_section_inter_lemmas} for completeness.


\begin{lemma}[Approximation Error]
\label{lemma:approx_error}
Let $W$ and $L$ denote the width and depth of the neural network architecture, respectively, and we have that $W\geq(S-1)d$. Then, for all $k \in [{0},K-1]$, we have
\begin{equation}
    \mathcal{E}^{\mathrm{approx}}_k = 0.
\end{equation}
\end{lemma}

The detailed proof of this lemma is provided in Appendix \ref{app:proof_approx_lemma}. We demonstrate in Lemma \ref{app:proof_approx_lemma} that if the width of the neural network $W$ is greater than the number of possible states $W\geq(S-1)d$, then it is possible to construct a neural network that is equal to any specified function defined on the $(S-1)d$ discrete points, where width $W$ is defined as the maximum number of neurons in a layer of the neural network. The discrete nature of the score function allows us to obtain this results. In the case of diffusion models the score functions were defined on a continuous input space. In such a case, this error would have been exponential in the data dimension as is obtained in \cite{jiao2023deep}.

\begin{lemma}[Statistical Error]
\label{lemma:stat_error}
    Let $n_{k}$ denote the number of samples used to estimate the score function at each diffusion step $t_{k}$. Then, then under Assumptions  \ref{ass:PL-condition} and \ref{ass:smoothness}, for all $k \in [{0},K-1]$, with probability at least $1-\gamma$, we have
\begin{equation}
    \mathcal{E}^{\mathrm{stat}}_k
    \leq \mathcal{O}\left(W^L\left((S-1)d+\frac{L}{W}\right)\sqrt\frac{{{\log}\left(\frac{2}{\gamma}\right)}}{{n_{k}}}\right)
\end{equation}
\end{lemma}

The proof of this lemma follows from utilizing the definitions of $s^{a}_{t}$ and $s^{b}_t$. Our novel analysis here allows us to avoid the exponential dependence on the neural networks parameters in the sample complexity. This is achieved without resorting to the quantile formulation and assuming access to the empirical risk minimizer of the score estimation loss which leads to a sample complexity of $\widetilde{\mathcal{O}}(\epsilon^{-2})$.
The detailed proof of this lemma is provided in Appendix \ref{appendix:stat_error}.

Next, we have the following lemma to bound the optimization error.

\begin{lemma}[Optimization Error]
\label{lemma:opt_error}
   Let $ n_{k}$ be the number of samples used to estimate the score function at diffusion step $k$. If the learning rate $0 \le \eta \le \frac{1}{\kappa}$, then under Assumptions  \ref{ass:PL-condition} and \ref{ass:smoothness}, for all $ k \in [0,K-1] $, the optimization error due to imperfect minimization of the training loss satisfies with probability at least $1-\gamma$
\begin{equation}
    \mathcal{E}^{\mathrm{opt}}_k \leq \mathcal{O} \left(W^L\left((S-1)d+\frac{L}{W}\right)\sqrt{\frac{ \log \left(\frac{2}{\gamma} \right)}{n_{k}}} \right)
\end{equation}
\end{lemma}
We leverage assumptions \ref{ass:PL-condition} and \ref{ass:smoothness} alongside our unique analysis to derive an upper bound on $\mathcal{E}_{\mathrm{opt}}$. The key here is our recursive analysis of the error at each stochastic gradient descent (SGD) step, which captures the error introduced by the finite number of SGD steps in estimating the score function, while prior works assumed no such error, treating the empirical loss minimizer as if it were known exactly. Here, the dependence on the number of SGD steps is implicit through $n_k$, since each SGD update at step $k$ uses a fresh sample and the total number of optimization iterations equals the effective sample size $n_k$.

Now, using the Lemmas \ref{lemma:approx_error}, \ref{lemma:stat_error}, and \ref{lemma:opt_error}, 
with probability at least $1-\gamma$, we have
\begin{align}
A_k &= \mathbb{E}_{\mathbf{x}_k \sim q_k} 
\bigl\| s_k(\mathbf{x}_k) - \hat{s}_{\theta,k}(\mathbf{x}_k) \bigr\|_2^2 \nonumber \\
&\lesssim\;
\underbrace{\mathcal{O}\!\left(W^L\left((S-1)d+\frac{L}{W}\right)\sqrt{\tfrac{\log(2/\gamma)}{n_k}}\right)}_{\text{Statistical Error}}
\\
&+ \underbrace{\mathcal{O}\!\left(W^L\left((S-1)d+\frac{L}{W}\right) \sqrt{\tfrac{\log(2/\gamma)}{n_k}}\right)}_{\text{Optimization Error}}
\nonumber \\
&\lesssim\; 
\mathcal{O}\!\left(W^L\left((S-1)d+\frac{L}{W}\right) \sqrt{\tfrac{\log(2/\gamma)}{n_k}}\right)
\label{eqn:Ak_aligned}
\end{align}

Plugging the results of Lemma \ref{lemma:approx_error},\ref{lemma:stat_error},\ref{lemma:opt_error} into (\ref{eqn:Ak_decomp}) gives us an upper bound on $B_{k}$. Plugging the upper bound on $B_{k}$ (\ref{eqn:Ak_aligned}) into (\ref{eqn: a_k_b_k_in}) yields the bound on $A_k$. Finally, plugging the bound on $A_k$ into (\ref{eq:final-kl-bound_main_text}) gives us the result required for Theorem \ref{thm:main_theorem}.




\section{Conclusion}
We investigate the sample complexity of training discrete-state diffusion models using sufficiently expressive neural networks. We derive a sample complexity bound of $\widetilde{\mathcal{O}}(\epsilon^{-2})$, which, to our knowledge, is the first such result under this setting. Notably, our analysis avoids exponential dependence on the data dimension. For comparison, existing works assume a bound over the efficiency of the trained network, while we have decomposed this error, using SGD for training through a dataset of finite size. \\
Recent work by \citet{sahoo2026orderoptimalsamplecomplexityrectified} shows that rectified flow matching (RFM) also achieves an order optimal $\widetilde{\mathcal{O}}(\epsilon^{-2})$ sample complexity, improving upon the $\widetilde{\mathcal{O}}(\epsilon^{-4})$ rate known for continuous flow matching \citep{gaur2025generativemodelingcontinuousflows}. Developing guarantees for discrete flow matching counterparts remains an important direction for future work, with the potential to improve practical performance. Additionally, a more refined treatment of approximation error arising from finite-capacity neural networks remains an open and promising avenue for further study.

\section{Acknowledgement}
This work is supported in part by the National Science Foundation under grant CCF-2149588.

\newpage
\bibliography{references}
\bibliographystyle{apalike}

\clearpage 
\section*{Checklist}
\begin{enumerate}

  \item For all models and algorithms presented, check if you include:
  \begin{enumerate}
    \item A clear description of the mathematical setting, assumptions, algorithm, and/or model. Yes
    \item An analysis of the properties and complexity (time, space, sample size) of any algorithm. Yes
    \item (Optional) Anonymized source code, with specification of all dependencies, including external libraries. Not Applicable
  \end{enumerate}

  \item For any theoretical claim, check if you include:
  \begin{enumerate}
    \item Statements of the full set of assumptions of all theoretical results. Yes
    \item Complete proofs of all theoretical results. Yes
    \item Clear explanations of any assumptions. Yes     
  \end{enumerate}

  \item For all figures and tables that present empirical results, check if you include:
  \begin{enumerate}
    \item The code, data, and instructions needed to reproduce the main experimental results (either in the supplemental material or as a URL). Not Applicable
    \item All the training details (e.g., data splits, hyperparameters, how they were chosen). Not Applicable
    \item A clear definition of the specific measure or statistics and error bars (e.g., with respect to the random seed after running experiments multiple times). Not Applicable
    \item A description of the computing infrastructure used. (e.g., type of GPUs, internal cluster, or cloud provider). Not Applicable
  \end{enumerate}

  \item If you are using existing assets (e.g., code, data, models) or curating/releasing new assets, check if you include:
  \begin{enumerate}
    \item Citations of the creator If your work uses existing assets. Not Applicable
    \item The license information of the assets, if applicable. Not Applicable
    \item New assets either in the supplemental material or as a URL, if applicable. Not Applicable
    \item Information about consent from data providers/curators. Not Applicable
    \item Discussion of sensible content if applicable, e.g., personally identifiable information or offensive content. Not Applicable
  \end{enumerate}

  \item If you used crowdsourcing or conducted research with human subjects, check if you include:
  \begin{enumerate}
    \item The full text of instructions given to participants and screenshots. Not Applicable
    \item Descriptions of potential participant risks, with links to Institutional Review Board (IRB) approvals if applicable. Not Applicable
    \item The estimated hourly wage paid to participants and the total amount spent on participant compensation. Not Applicable
  \end{enumerate}

\end{enumerate}

\newpage
\appendix
\onecolumn




\aistatstitle{Supplementary material}

\section{Training and Sampling Algorithm} \label{alg_not}

In this section, we provide a practical training and sampling algorithm for discrete-state diffusion models. The forward process (Step 6 in Algorithm \ref{alg:disc_score_training}) is based on Proposition 1 of  \citet{zhang2025convergencescorebaseddiscretediffusion} and the training happens through minimizing the empirical score entropy loss Equation (\ref{eqn:ERM_loss_assumption_1}). Algorithm \ref{alg:uniformization_sampling} is from \cite{zhang2025convergencescorebaseddiscretediffusion}.
\begin{algorithm}[H]
\caption{A practical training procedure for Discrete-State Score-based Diffusion Models via Score Entropy Loss}
\label{alg:disc_score_training}
\begin{algorithmic}[1]
\Require Dataset $\mathcal{D}$ of samples $x_0 \in [S]^d$ (alphabet size $S$); diffusion time $T$; step-size $h$; early-stop $\delta \ge 0$; number of steps $K$ with $T = Kh + \delta$; network $s_\theta(x,t) \in \mathbb{R}^{d \times (S-1)}$; learning rate $\eta$; batch size $B$; epochs $E$.
\State \textbf{Per-token forward kernel :} $P_{0,t}= \tfrac{1}{S}(1-e^{-t})\,\ones\ones^\top + e^{-t}I$.
\For{epoch $=1$ to $E$}
  \For{$k=0$ \textbf{to} $K-1$}
    \State Sample minibatch $\{x_0^i\}_{i=1}^B \sim \mathcal{D}$.
    \State Sample $t \sim \Unif\!\big([kh+\delta,\,(k+1)h+\delta]\big)$.
    \State \textbf{Forward corruption (factorized CTMC):} for each sample $i$ and coord $j$,
           $x_t^{i,j} \sim \Categorical\!\big(P_{0,t}[\,x_0^{i,j},:\,]\big)$.
    \State \textbf{Build ratio targets:} for each $(i,j)$ and each $a \in [S],\, a \neq x_t^{i,j}$,
           $r_t^{i}(j,a)=\dfrac{P_{0,t}[x_0^{i,j},a]}{P_{0,t}[x_0^{i,j},x_t^{i,j}]}\,$.
    \State \textbf{Query score network frozen at the upper-limit of each interval:} $t'=(k+1)h+\delta$; compute $\hat{s}^i = s_\theta(x_t^i,\, t')$.
    \State \textbf{Score Entropy Loss:}
    \Statex \(
      \displaystyle
      L(\theta)=\frac{1}{B}\sum_{i=1}^B \frac{1}{S}
      \sum_{j=1}^d \sum_{\substack{a=1\\ a\neq x_t^{i,j}}}^{S}
      \Big(-\,r_t^{i}(j,a)\log \hat{s}^i(j,a) + \hat{s}^i(j,a)\Big).
    \)
    \State \textbf{Update:} $\theta \leftarrow \theta - \eta\,\nabla_\theta L(\theta)$.
  \EndFor
\EndFor
\end{algorithmic}
\end{algorithm}

\begin{algorithm}[H]
\caption{Generative Reverse Process Simulation through Uniformization}
\label{alg:uniformization_sampling}
\begin{algorithmic}[1]
\Require Learned discrete score functions $\{\hat s_{T-kh}\}_{k=0}^{K-1}$, total time $T$, step $h>0$, and $\delta = T - Kh \ge 0$.
\State Draw $z_{0} \sim \pi^{d}$.
\For{$k = 0$ \textbf{to} $K-1$}
  \State Set $\lambda_k \gets \max_{x \in [S]^d} \{\, \hat s_{T-kh}(x)_{x} \,\}$.
  \State Draw $N \sim \operatorname{Poisson}(\lambda_k h)$.
  \State Set $y_{0} \gets z_{k}$.
  \For{$j = 0$ \textbf{to} $N-1$}
    \Statex \(
      \displaystyle
      y_{j+1} =
      \begin{cases}
        y_j^{\backslash i} \odot \hat y^{\,i}, &
        \text{w.p. } \dfrac{\hat s_{T-kh}(y_j)_{i,\hat y^{\,i}}}{\lambda_k},\
        1 \le i \le d,\ \hat y^{\,i} \in [S],\ \hat y^{\,i} \neq y_j^{\,i},\\[0.8em]
        y_j, &
        \text{w.p. } 1 - \displaystyle\sum_{i=1}^{d}\ \sum_{\hat y^{\,i} \neq y_j^{\,i}}
        \dfrac{\hat s_{T-kh}(y_j)_{i,\hat y^{\,i}}}{\lambda_k}.
      \end{cases}
    \)
  \EndFor
  \State $z_{k+1} \gets y_{N}$.
\EndFor
\State \textbf{Output:} a sample $z_{K}$ from $p_{T-\delta}$.
\Statex \emph{Notation:} $\hat s_{t}(x)_{x} := \sum_{i=1}^{d}\sum_{\hat x^{\,i}\neq x^{i}}
\hat s_{t}(x)_{i,\hat x^{\,i}}$, and $x^{\backslash i}\!\odot\! \hat x^{\,i}$ replaces coordinate $i$ of $x$ with $\hat x^{\,i}$.
\end{algorithmic}
\end{algorithm}

\noindent\textbf{Discussion on $C$} 
Since we have access to uniform bounds for the true score functions from Lemma \ref{lemma: score_uniform_bound}, which depend on $B$, we can apply score clipping during training to ensure that the learned score functions are reliable. Specifically, we enforce the following conditions:

\begin{equation}
\max_{\mathbf{x} \in \mathcal{X},\, k \in \{0, \ldots, K-1\}}
\left\| \hat{s}_{(k+1)h}(\mathbf{x}) \right\|_{\infty}
\le \frac{3}{2}B.
\end{equation}
As a result, we can choose that 
\[
C = \frac{3}{2}B
\]
\section{Supplmentary Lemmas}\label{sec:supp_lemm}

\textbf{Definition} (Bregman divergence). Let $\phi$ be a strictly convex function defined on a convex set $\mathcal{S} \subset \mathbb{R}^d$ ($d \in \mathbb{N}_+$) and $\phi$ is differentiable. The Bregman divergence $D_\phi(x\|y) : \mathcal{S} \times \mathcal{S} \to \mathbb{R}_+$ is defined as
\[
D_\phi(x\|y) = \phi(x) - \phi(y) - \nabla \phi(y)^\top (x - y).
\]

In particular, the generalized I-divergence
\[
D_I(x\|y) = \sum_{i=1}^d \left[-x^i + y^i + x^i \log \frac{x^i}{y^i} \right]
\]
is generated by the negative entropy function $I(x) = \sum_{i=1}^d x^i \log x^i$.

The Bregman divergence does not satisfy the triangle inequality. However, when the domain of the negative entropy is restricted to a closed box contained in $\mathbb{R}_+^d$, we have the following Lemma from Proposition 3 of \cite{zhang2025convergencescorebaseddiscretediffusion}, which provides an analogous form of the triangle inequality. \\
\begin{lemma}\label{lemma: bregman_decomposition}
Let the negative entropy function \( I(x) = \sum_{i=1}^d x^i \log x^i \) be defined on \( [a, b]^d \subset \mathbb{R}_{+}^d \) for constants \( 0 < a < b \). Then for all \( x, y, z \in [a, b]^d \), we have:


\begin{enumerate}[label=(\roman*),ref=(\roman*),leftmargin=*,align=left]
\item\label{eq:bregman-decomposition_1}
\begin{equation*}
D_I(x \| y) \;\leq\; \frac{1}{a}\,\|x - z\|_2^2 \;+\; \frac{2b}{a}\, D_I(z \| y).
\end{equation*}

\item\label{eq:bregman-decomposition_2}
\begin{equation*}
D_I(x \| y) \;\leq\; \frac{1}{a}\,\|x-z\|_2^2 \;+\; \frac{b}{a^2}\,\|z-y\|_2^2
\end{equation*}
\end{enumerate}


\end{lemma}

\begin{proof}
Let \( I(x) = \sum_{i=1}^d x^i \log x^i \), and note that its Hessian is given by
\begin{equation}
\nabla^2 I(x) = \text{diag}\left( \frac{1}{x^1}, \ldots, \frac{1}{x^d} \right).
\end{equation}
Since \( x \in [a, b]^d \), we have
\begin{equation}
\frac{1}{b} I_d \preceq \nabla^2 I(x) \preceq \frac{1}{a} I_d.
\end{equation}
This implies that \( I \) is \( \frac{1}{b} \)-strongly convex and its gradient \( \nabla I \) is \( \frac{1}{a} \)-Lipschitz smooth.

By the second-order Taylor expansion of the Bregman divergence, there exists \( \theta \in [0,1] \) such that
\begin{equation}
D_I(x \| y) = \frac{1}{2} (x - y)^\top \nabla^2 I(y + \theta(x - y)) (x - y).
\end{equation}
Using \( \nabla^2 I(x) \preceq \frac{1}{a} I_d \), we obtain
\begin{equation}
D_I(x \| y) \leq \frac{1}{2a} \|x - y\|_2^2.
\label{eqn: lemma_bregman_upper_bound}
\end{equation}

Now apply the triangle inequality to the right-hand side of (\ref{eqn: lemma_bregman_upper_bound}):
\begin{equation}
\|x - y\|_2^2 \leq 2\|x - z\|_2^2 + 2\|z - y\|_2^2.
\end{equation}
Substituting into the previous inequality (\ref{eqn: lemma_bregman_upper_bound}) gives
\begin{equation}
D_I(x \| y) \leq \frac{1}{a} \|x - z\|_2^2 + \frac{1}{a} \|z - y\|_2^2.
\label{eqn: lemma_bregman_triangle}
\end{equation}

Since \( I \) is \( \frac{1}{b} \)-strongly convex, we have
\begin{equation}
D_I(z \| y) \geq \frac{1}{2b} \|z - y\|_2^2 \quad \Rightarrow \quad \|z - y\|_2^2 \leq 2b \cdot D_I(z \| y).
\end{equation}
Substituting this into (\ref{eqn: lemma_bregman_triangle}) yields
\begin{equation}
D_I(x \| y) \leq \frac{1}{a} \|x - z\|_2^2 + \frac{2b}{a} \cdot D_I(z \| y).
\end{equation}
This completes the proof of the first part (\ref{eq:bregman-decomposition_1}) of the Lemma. 

Because $\nabla^2 I(x) = \mathrm{diag}(1/x^1, \ldots, 1/x^d) \preceq \tfrac{1}{a} I_d$ on $[a,b]^d$, 
the function $I$ is $\tfrac{1}{a}$-smooth. Hence, for $z,y$,
\begin{equation}
D_I(z \| y) \;\leq\; \frac{1}{2a}\,\|z-y\|_2^2.
\label{eq:smooth-bound}
\end{equation}

Substituting this in (\ref{eq:bregman-decomposition_1}), we obtain the following:
\begin{align}
D_I(x \| y)
&\;\leq\; \frac{1}{a}\,\|x-z\|_2^2 \;+\; \frac{2b}{a}\, D_I(z \| y) \nonumber \\
&\;\leq\; \frac{1}{a}\,\|x-z\|_2^2 \;+\; \frac{2b}{a}\cdot \frac{1}{2a}\,\|z-y\|_2^2 \nonumber \\
&\;\leq\; \frac{1}{a}\,\|x-z\|_2^2 \;+\; \frac{b}{a^2}\,\|z-y\|_2^2.
\label{eq:two-l2-bound}
\end{align}

It is generally a common practice to use the bounds $[1/C, C]$ instead of $[a, b]$ from Proposition 3 in \citet{zhang2025convergencescorebaseddiscretediffusion}. By substituting $a = 1/C$ and $b = C$ in (\ref{eq:two-l2-bound}), we obtain the following bound:

\begin{equation}
D_I(x \,\|\, y) \;\le\; C \,\|x - z\|_2^2 \;+\; C^3 \,\|z - y\|_2^2.
\end{equation}

This proves the second part (\ref{eq:bregman-decomposition_2}) of the Lemma.
\end{proof} 

This completes the proof of Lemma \ref{lemma: bregman_decomposition}

\begin{lemma}
\label{lem:kl-divergence_path_measure}
The KL divergence between the true and approximate path measures of the reverse process, both starting from $q_T$, is \citep{zhang2025convergencescorebaseddiscretediffusion}
\begin{equation}
D_{\mathrm{KL}}(\mathbb{Q}\,\|\,\mathbb{P}^{q_T})
= \sum_{k=0}^{K-1} \int_{kh+\delta}^{(k+1)h+\delta} 
\mathbb{E}_{\mathbf{x}_t \sim q_t} \sum_{i=1}^d \sum_{\hat{x}^i_t \neq x^i_t} 
Q_t^{\mathrm{tok}}(\hat{x}^i_t, x^i_t) \,
D_I\!\left( s_t(\mathbf{x}_t)_{i,\hat{x}^i_t} \,\|\, 
\hat{s}_{(k+1)h+\delta}(\mathbf{x}_t)_{i,\hat{x}^i_t} \right) dt.
\end{equation}
\end{lemma}


\begin{lemma}[Score bound -  follows from Lemma 4 of  \citet{zhang2025convergencescorebaseddiscretediffusion}]
\label{lemma: score_uniform_bound}
Suppose Assumption \ref{full_support} holds. Then for all $t \in [0,T]$, $\mathbf{x} \in \mathcal{X}$, 
$i \in [d]$ and $\hat{x}^i \neq x^i \in [S]$, we have
\[
\dfrac{1}{B}\;\leq\; s_t(\mathbf{x})_{i,\hat{x}^i} \;\leq\; B.
\]
\end{lemma}
\begin{proof}
The Proof of Lemma \ref{lemma: score_uniform_bound} follows from the Proof of Lemma 4 of \citet{zhang2025convergencescorebaseddiscretediffusion}
Define the kernel function
\[
g_{\mathbf{w}}(t)
= \frac{1}{S^{d}} \prod_{i=1}^{d} \Bigl[\,1 + e^{-t}\bigl(-1 + S\cdot \mathbf{1}\{w^{i}\equiv 0 \ (\mathrm{mod}\ S)\}\bigr)\Bigr],
\qquad \mathbf{w}\in \mathbb{Z}^{d},\ t\ge 0 .
\]

From Proposition 1 of \citet{zhang2025convergencescorebaseddiscretediffusion}, the transition probability of the forward process can be written as
\begin{align*}
q_{t|s}(\mathbf{x}_{t}\mid \mathbf{x}_{s})
&= \prod_{i=1}^{d} q^{i}_{t|s}(x^{i}_{t}\mid x^{i}_{s})
 = \prod_{i=1}^{d} P^{0}_{s,t}(x^{i}_{s},x^{i}_{t})\\
&= \frac{1}{S^{d}} \prod_{i=1}^{d} \Bigl[\,1 + e^{-(t-s)}\bigl(-1 + S\cdot \delta\{x^{i}_{t},x^{i}_{s}\}\bigr)\Bigr],
\qquad \forall\, t>s\ge 0 .
\end{align*}

Hence \(q_{t|0}(\mathbf{y}\mid \mathbf{x}) = g_{\mathbf{y}-\mathbf{x}}(t)\) for \(t>0\).
Thus, for all \(t\in (0,T]\), \(\mathbf{x}\in \mathcal{X}\), \(i\in [d]\), and \(x^{i}\neq \hat{x}^{i}\in [S]\),
\begin{align*}
s_{t}(\mathbf{x})_{i,\hat{x}^{i}}
&= \frac{q_{t}(\mathbf{x}^{\backslash i}\odot \hat{x}^{i})}{q_{t}(\mathbf{x})}
 = \frac{q_{t}\bigl(\mathbf{x}+(\hat{x}^{i}-x^{i})\mathbf{e}_{i}\bigr)}{q_{t}(\mathbf{x})}\\
&= \frac{\sum_{\mathbf{y}\in \mathcal{X}} q_{0}(\mathbf{y})\, q_{t|0}\bigl(\mathbf{x}+(\hat{x}^{i}-x^{i})\mathbf{e}_{i}\mid \mathbf{y}\bigr)}{q_{t}(\mathbf{x})}
 = \frac{\sum_{\mathbf{y}\in \mathcal{X}} q_{0}(\mathbf{y})\, g_{\mathbf{x}+(\hat{x}^{i}-x^{i})\mathbf{e}_{i}-\mathbf{y}}(t)}{q_{t}(\mathbf{x})}\\
&= \frac{\sum_{\mathbf{y}\in \mathcal{X}} q_{0}\bigl(\mathbf{y}+(\hat{x}^{i}-x^{i})\mathbf{e}_{i}\bigr)\, g_{\mathbf{x}-\mathbf{y}}(t)}{q_{t}(\mathbf{x})}
 = \sum_{\mathbf{y}\in \mathcal{X}} \frac{q_{0}(\mathbf{y})\, q_{t|0}(\mathbf{x}\mid \mathbf{y})}{q_{t}(\mathbf{x})}\,
    \frac{q_{0}\bigl(\mathbf{y}+(\hat{x}^{i}-x^{i})\mathbf{e}_{i}\bigr)}{q_{0}(\mathbf{y})}\\
&= \dfrac{1}{B} \;\le\; \mathbb{E}_{\mathbf{y}\sim q_{0|t}(\cdot\mid \mathbf{x})}
    \Biggl[\frac{q_{0}\bigl(\mathbf{y}+(\hat{x}^{i}-x^{i})\mathbf{e}_{i}\bigr)}{q_{0}(\mathbf{y})}\Biggr]
 \;\le\; B,
\end{align*}
where the last inequality uses Assumption \ref{full_support}. The sum \(\mathbf{y}+(\hat{x}^{i}-x^{i})\mathbf{e}_{i}\) is interpreted in modulo-\(S\) sense. \qed

\end{proof}

\begin{lemma}[Score movement bound - Lemma 5 of \citet{zhang2025convergencescorebaseddiscretediffusion}]
\label{lemma: score_movement_bound_full_support}
Suppose Assumption \ref{full_support} holds. Let $\delta = 0$. Then for all $k \in \{0,1,\dots,K-1\}$, 
$t \in [kh,(k+1)h]$, $\mathbf{x} \in \mathcal{X}$, $i \in [d]$, and $\hat{x}^i \neq x^i \in [S]$, we have
\[
\bigl| s_t(\mathbf{x})_{i,\hat{x}^i} - s_{(k+1)h}(\mathbf{x})_{i,\hat{x}^i} \bigr|
\;\lesssim\;
\left[ \frac{1}{1 - e^{-(k+1)h}} + S \right] \kappa_i h.
\]
\end{lemma}

\begin{lemma}[Theorem 26.5 of \citet{shalev2014understanding}]
\label{lemma:thm26_5shalev}
Assume that for all $z$ and $h \in \mathcal{H}$ we have that $|\ell(h, z)| \leq c$.
Then,
\begin{enumerate}
\item With probability of at least $1 - \delta$, for all $h \in \mathcal{H}$,
\begin{align}
    L_D(h) - L_S(h) \leq 2 \underset{S' \sim D^m}{\mathbb{E}} R(\ell \circ \mathcal{H} \circ S') + c\sqrt{\frac{2 \ln(2/\delta)}{m}}.
\end{align}
In particular, this holds for $h = \text{ERM}_{\mathcal{H}}(S)$.

\item With probability of at least $1 - \delta$, for all $h \in \mathcal{H}$,
\begin{align}
    L_D(h) - L_S(h) \leq 2R(\ell \circ \mathcal{H} \circ S) + 4c\sqrt{\frac{2 \ln(4/\delta)}{m}}.
\end{align}
In particular, this holds for $h = \text{ERM}_{\mathcal{H}}(S)$.

\item For any $h^*$, with probability of at least $1 - \delta$,
\begin{align}
    L_D(\text{ERM}_{\mathcal{H}}(S)) - L_D(h^*) \leq 2R(\ell \circ \mathcal{H} \circ S) + 5c\sqrt{\frac{2 \ln(8/\delta)}{m}}.
\end{align}
\end{enumerate}
\end{lemma}

\begin{lemma}[Extension of Massart's Lemma]\label{lemm:mass_ex}
    Let $\Theta^{''}$ be a finite function class. Then, for any $\theta \in \Theta^{''}$, we have
    \begin{align}
        \mathbb{E}_{\sigma} \left[ \max_{\theta \in \Theta^{''}}\sum_{i=1}^{n}f(\theta){\sigma_{i}} \right] 
        \le ||f(\theta)||_{2} \le   (BW)^L\!\left(d+\frac{L}{W}\right)
    \end{align}
    where $\sigma_{i}$ are i.i.d random variables such that $\mathbb{P}(\sigma_{i} = 1) = \mathbb{P}(\sigma_{i} = -1) =\frac{1}{2}$. 
    We get the second inequality by denoting $L$ as the number of layers in the neural network, $W$ and $B$ a constant such all parameters of the neural network upper bounded by $B$. $d$ is the dimension of the function $f(\theta)$. 
\end{lemma}

\begin{proof}
Let \(h_0=x\), and for \(\ell=0,\dots,L-1\) define the layer recursion
\[
h_{\ell+1}=\sigma(W_\ell h_\ell + b_\ell),
\]
where \(W_\ell\in\mathbb{R}^{n_{\ell+1}\times n_\ell}\), \(b_\ell\in\mathbb{R}^{n_{\ell+1}}\), and \(n_\ell\le W\) for hidden layers. We work with the \(\ell_\infty\) operator norm:
\[
\|W_\ell\|_\infty=\max_i \sum_j |(W_\ell)_{ij}| \;\le\; B\,n_\ell \;\le\; BW=\alpha.
\]
Since \(\sigma\) is \(1\)-Lipschitz with \(\sigma(0)=0\), we have \(\|\sigma(u)\|_\infty\le\|u\|_\infty\) and thus
\[
\|h_{\ell+1}\|_\infty
\;\le\; \|W_\ell\|_\infty\,\|h_\ell\|_\infty + \|b_\ell\|_\infty
\;\le\; \alpha\,\|h_\ell\|_\infty + B.
\]
With \(\|h_0\|_\infty\le d\), iterating this affine recursion yields the standard geometric-series bound
\[
\|h_L\|_\infty \;\le\; \alpha^L d \;+\; B\sum_{i=0}^{L-1}\alpha^{\,i}
\;=\; \alpha^L d \;+\; B\,\frac{\alpha^{L}-1}{\alpha-1}
\quad(\alpha\neq 1),
\]
and for \(\alpha=1\), \(\|h_L\|_\infty\le d+BL\). The scalar output \(f(x)\) is either a coordinate of \(h_L\) or obtained by applying the same \(1\)-Lipschitz activation to a linear form of \(h_L\); in either case, \(|f(x)|\le\|h_L\|_\infty\), giving the stated bound.

For the \(\alpha\ge 1\) simplification, use \(\sum_{i=0}^{L-1}\alpha^i \le L\alpha^{L-1}\) to obtain
\[
|f(x)| \;\le\; \alpha^L d + B L \alpha^{L-1}
\;=\; (BW)^L\!\left(d+\frac{L}{W}\right).
\]
For \(\alpha<1\), since \(\alpha^i\le 1\), \(\sum_{i=0}^{L-1}\alpha^i\le L\) and hence \(|f(x)|\le d+BL\). Note that this is a tighter bound that the one for $\alpha \ge 1$. However, we retain the bound for $\alpha \ge 1$ to make our result hold for both cases. 
\end{proof}

\begin{lemma}[Theorem 4.6 of \cite{gower2021sgd}]
\label{lemma:thm4_6gower}
Let $f$ be $L$-smooth. Assume $f \in PL(\mu)$ and $g \in ER(\rho)$. Let $\gamma_k = \gamma \leq \frac{1}{1+2\rho/\mu} \cdot \frac{1}{L}$, for all $k$, then SGD update given by $x^{k+1} = x^k  - \gamma^k g(x^k)$ converges as follows
\begin{align}
    \mathbb{E}[f(x^k) - f^*] \leq (1 - \gamma\mu)^k [f(x^0) - f^*] + \frac{L\gamma\sigma^2}{\mu}.
\end{align}
Hence, given $\varepsilon > 0$ and using the step size $\gamma = \frac{1}{L} \min \left\{ \frac{\mu\varepsilon}{2\sigma^2}, \frac{1}{1+2\rho/\mu} \right\}$ we have that
\begin{align}
k \geq \frac{L}{\mu} \max \left\{ \frac{2\sigma^2}{\mu\varepsilon}, 1 + \frac{2\rho}{\mu} \right\} \log \left( \frac{2(f(x^0) - f^*)}{\varepsilon} \right)
\quad \Longrightarrow \quad \mathbb{E}[f(x^k) - f^*] \leq \varepsilon.
\end{align}
\end{lemma}

\begin{lemma}[Discrete--continuous score-error: absolute $O(h)$ gap]\label{thm:disc-vs-cont-abs}
Let $(X_t)_{t\in[\delta,T]}$ be a continuous-time Markov chain on a finite state space
$\Omega$ with generator $Q=(q(x,y))_{x,y\in\Omega}$, and write $q_t$ for the law of $X_t$.
Fix a stepsize $h>0$ and the grid $t_k:=\delta+kh$ for $k=0,\dots,K-1$ with
intervals $I_k=[t_k,t_{k+1}]$ and $Kh=T-\delta$.
For each $k$, define
\[
f_k(x)\;:=\;D\!\left(s_{t_k}(x)\,\big\|\,\hat s_{t_k}(x)\right),
\]
and assume a 
finite uniform departure rate
\begin{equation}\label{eq:lambda-abs}
\lambda \;:=\; \sup_{x\in\Omega}\sum_{y\neq x} q(x,y) \;<\;\infty.
\end{equation}
Let $S\ge 1$ denote the (harmless) scaling constant.
Define the continuous- and discrete-time score-error terms
\[
\mathrm{CT}:=\frac1S\sum_{k=0}^{K-1}\int_{I_k}\mathbb{E}_{q_t}\!\big[f_k(X)\big]\,dt,
\qquad
\mathrm{DT}:=\frac1S\sum_{k=0}^{K-1} h\,\mathbb{E}_{q_{t_k}}\!\big[f_k(X)\big]
\]
Then the absolute difference satisfies
\begin{equation}\label{eq:abs-gap-abs}
\big|\mathrm{CT}-\mathrm{DT}\big| \;\le\; \frac{(S-1)dC\,\lambda\,T\,h}{S}
\end{equation}
\end{lemma}

\begin{proof}
Fix $k$ and set $g_k(t):=\mathbb{E}_{q_t}[f_k(X)]$ for $t\in I_k$.
By the Kolmogorov forward equation $(\dot q_t = Q^\top q_t)$,
\[
\frac{d}{dt}\,g_k(t)
=\sum_x (Q^\top q_t)(x)f_k(x)
=\sum_x q_t(x)\,(Qf_k)(x)
=\mathbb{E}_{q_t}\!\big[(Qf_k)(X)\big],
\]
where $(Qf_k)(x)=\sum_{y\neq x} q(x,y)\big(f_k(y)-f_k(x)\big)$.
Using (\ref{lemma: score_uniform_bound}) and the triangle inequality,
\[
\big|(Qf_k)(x)\big|
\;\le\; \sum_{y\neq x} q(x,y)\,\big|f_k(y)-f_k(x)\big|
\;\le\; \sum_{y\neq x} q(x,y)\cdot 2(S-1)
dC\;\le\; 2(S-1)dC\lambda.
\]
Hence $\sup_{t\in I_k}\big|g_k'(t)\big|\le 2(S-1)dC\lambda$. From, Lemma \ref{lem:left-rectangle} we have
\[
\left|\int_{t_k}^{t_{k+1}} g_k(t)\,dt - h\,g_k(t_k)\right|
\;\le\; \frac{h^2}{2}\,\sup_{u\in I_k}\big|g_k'(u)\big|
\;\le\; 2(S-1)dC\lambda h^2
\]
Summing over $k=0,\dots,K-1$ gives
\[
\left|\sum_{k=0}^{K-1}\int_{I_k}\!\mathbb{E}_{q_t}[f_k]\,dt
-\sum_{k=0}^{K-1} h\,\mathbb{E}_{q_{t_k}}[f_k]\right|
\;\le\; \sum_{k=0}^{K-1} (S-1)dC\lambda h^2
\;=\; (S-1)dC\lambda Kh^2
\;\le\; (S-1)dC\lambda T h,
\]
since $Kh=T-\delta\le T$. Dividing both sides by $S$ yields (\ref{eq:abs-gap-abs}).
\end{proof}

If $\|s_t\|_\infty,\|\hat s_t\|_\infty\le C$ on $[\delta,T]$ and $D$ is a Bregman divergence
whose generator is smooth on the corresponding bounded set, then $M\le c\,C^2$ for a constant $c$.
For single–coordinate refresh dynamics in $d$ dimensions with uniform symbol choice,
$\lambda\le d\,(S-1)/S$. Consequently, (\ref{eq:abs-gap-abs}) gives
$\big|\mathrm{CT}-\mathrm{DT}\big|\le (c\,C_1^2)\,d\,T\,h/S$.

\begin{lemma}[First-order quadrature remainder (left rectangle rule)]
\label{lem:left-rectangle}
Let $g:[a,b]\to\mathbb{R}$ be absolutely continuous with $g'\in L^\infty([a,b])$.
Then
\begin{equation}\label{eq:left-rect-identity}
\int_a^b g(t)\,dt \;-\; (b-a)\,g(a)
\;=\; \int_a^b (b-u)\,g'(u)\,du,
\end{equation}
and consequently
\begin{equation}\label{eq:left-rect-bound}
\Bigl|\int_a^b g(t)\,dt \;-\; (b-a)\,g(a)\Bigr|
\;\le\; \frac{(b-a)^2}{2}\,\|g'\|_{\infty,[a,b]}.
\end{equation}
Moreover, the constant $1/2$ is sharp (equality holds for affine $g(t)=\alpha+\beta t$).
\end{lemma}

\begin{proof}
By absolute continuity and the fundamental theorem of calculus,
$g(t)=g(a)+\int_a^t g'(u)\,du$ for all $t\in[a,b]$. Integrating both sides in $t$ gives
\[
\int_a^b g(t)\,dt
= (b-a)\,g(a) + \int_a^b \!\left(\int_a^t g'(u)\,du\right)\!dt.
\]
By Tonelli/Fubini (justified since $g'\in L^\infty$),
\[
\int_a^b \!\left(\int_a^t g'(u)\,du\right)\!dt
= \int_a^b \!\left(\int_u^b dt\right) g'(u)\,du
= \int_a^b (b-u)\,g'(u)\,du,
\]
which proves Equation (\ref{eq:left-rect-identity}).
Taking absolute values and using $0\le b-u\le b-a$,
\[
\Bigl|\int_a^b (b-u)\,g'(u)\,du\Bigr|
\le \|g'\|_{\infty,[a,b]} \int_a^b (b-u)\,du
= \|g'\|_{\infty,[a,b]} \,\frac{(b-a)^2}{2},
\]
which is (\ref{eq:left-rect-bound}). Sharpness follows by direct calculation for $g(t)=\alpha+\beta t$,
where the error equals $\beta(b-a)^2/2$.
\end{proof}

[Right rectangle, midpoint, and nonuniform steps]\label{rem:variants}
(1) \emph{Right rectangle.}
An identical argument with $g(b)$ in place of $g(a)$ yields
\[
\int_a^b g(t)\,dt - (b-a)\,g(b)
= -\int_a^b (u-a)\,g'(u)\,du,
\qquad
\Bigl|\int_a^b g(t)\,dt - (b-a)\,g(b)\Bigr|
\le \frac{(b-a)^2}{2}\,\|g'\|_\infty .
\]

(2) \emph{Midpoint rule (needs a second derivative).}
If $g'$ is absolutely continuous and $g''\in L^\infty([a,b])$, then
\[
\Bigl|\int_a^b g(t)\,dt - (b-a)\,g\!\left(\tfrac{a+b}{2}\right)\Bigr|
\le \frac{(b-a)^3}{24}\,\|g''\|_\infty.
\]
We do not use this in the theorem since we only assume a bound on $g'$.

(3) \emph{Nonuniform mesh.}
For intervals $I_k=[t_k,t_{k+1}]$ with $h_k=t_{k+1}-t_k$ and $g_k$ absolutely continuous on $I_k$,
\[
\left|\int_{I_k} g_k(t)\,dt - h_k\,g_k(t_k)\right|
\le \frac{h_k^2}{2}\,\|g_k'\|_{\infty,I_k},
\qquad
\sum_k \left|\int_{I_k} g_k - h_k g_k(t_k)\right|
\le \frac{1}{2}\sum_k h_k^2 \|g_k'\|_{\infty,I_k}.
\]
If $h_{\max}:=\max_k h_k$, this implies
$\sum_k \left|\int_{I_k} g_k - h_k g_k(t_k)\right|
\le \tfrac{h_{\max}}{2}\sum_k h_k \|g_k'\|_{\infty,I_k}$.

\paragraph{How this plugs into the theorem.}
In the proof of Theorem~\ref{thm:disc-vs-cont-abs}, set $a=t_k$, $b=t_{k+1}$ and
$g_k(t):=\mathbb{E}_{q_t}[f_k(X)]$.
From the CTMC calculus we had $g_k'(t)=\mathbb{E}_{q_t}[(Qf_k)(X)]$ and
$\|g_k'\|_{\infty,I_k}\le 2(S-1)dC\lambda$.
Applying Lemma~\ref{lem:left-rectangle} gives the per-interval bound
\[
\left|\int_{t_k}^{t_{k+1}} g_k(t)\,dt - h\,g_k(t_k)\right|
\le \frac{h^2}{2}\cdot 2(S-1)dC\lambda
= M\lambda h^2.
\]
Summing over $k$ and using $Kh=T-\delta\le T$ yields
$\big|\mathrm{CT}-\mathrm{DT}\big|
\le (2(S-1)dC\lambda/S)\,Kh^2 \le (2(S-1)dC\lambda/S)\,T h$,
which is exactly the absolute $O(h)$ gap.

\section{Intermediate Lemmas} \label{appendix_section_inter_lemmas}
In this section, we present the proofs of intermediate lemmas used to bound the approximation error, statistical error, and optimization error in our analysis.

\subsection{Bounding the Approximation error}\label{appendix: approx_error}
\textbf{Recall Lemma \ref{lemma:approx_error} (Approximation Error.)}
\textit{Let $W$ and $D$ denote the width and depth of the neural network architecture, respectively, and let $Sd$ represent the input data dimension. Then, for all $k \in [{0},K-1]$ we have
\begin{equation}
\mathcal{E}^{\mathrm{approx}}_k  = 0
\end{equation}}

\label{app:proof_approx_lemma}


\begin{proof}
We construct a network explicitly such that it matches any given function defined on a fixed set of points. The proof has two parts: a base construction of depth $2$ (one hidden layer) and an inductive padding by identity layers that preserves the values on $X$ while increasing depth without increasing width.

We give a constructive proof. The idea is (i) build a depth-2 interpolant using only $S$ of the $W$ hidden units and zero out the remaining $W-S$, then
(ii) increase depth by inserting identity blocks that act as the identity on the finite set of hidden codes, still using width $W$.

\smallskip
\noindent\textbf{Step 0 (Distinct scalar coordinates).}
Since $x_1,\dots,x_S$ are distinct, choose $a\in\mathbb{R}^d$ so that $t_i:=a^\top x_i$ are pairwise distinct.
Relabel so $t_1<t_2<\cdots<t_S$.

\smallskip
\noindent\textbf{Step 1 (Depth-2 interpolant of width $W\ge S$).}
Pick biases $b_1,\dots,b_S$ such that
\[
b_1<t_1,\qquad b_k\in\Big(\tfrac{t_{k-1}+t_k}{2},\,t_k\Big)\ \text{ for }k=2,\dots,S.
\]
For $\alpha>0$ define the $S\times S$ matrix $A^{(\alpha)}$ with entries
\[
A^{(\alpha)}_{ik}:=\phi\!\big(\alpha(t_i-b_k)\big),\qquad 1\le i,k\le S.
\]
As in the triangular-limit argument, the rescaled matrix
\[
\widetilde A^{(\alpha)}:=\frac{A^{(\alpha)}-\phi_- \mathbf{1}\mathbf{1}^\top}{\phi_+-\phi_-}
\]
converges entrywise to the unit lower-triangular matrix $L$ as $\alpha\to\infty$, hence is invertible for some finite $\alpha^\star>0$.
Let $y=(f(x_1),\dots,f(x_S))^\top$ and choose any $c_0\in\mathbb{R}$.
Solve for $c=(c_1,\dots,c_S)^\top$ in
\[
A^{(\alpha^\star)}c = y - c_0\mathbf{1}.
\]
Now define a one-hidden-layer network with width $W$ by using the first $S$ hidden units exactly as above and \emph{setting the remaining $W-S$ hidden units to arbitrary parameters but with output weights equal to zero}. Concretely,
\[
N_2(x)
:= c_0 + \sum_{k=1}^{S} c_k\,\phi\!\big(\alpha^\star(a^\top x - b_k)\big)
    + \sum_{k=S+1}^{W} 0\cdot \phi(\,\cdot\,).
\]
Then $N_2(x_i)=f(x_i)$ for all $i$.
Thus we have a depth-2 interpolant of width $W$.

\smallskip
\noindent\textbf{Step 2 (Identity-on-a-finite-set block with width $W\ge S$).}
We show that for any finite set $U=\{u_1,\dots,u_S\}\subset\mathbb{R}^m$ there exists a one-hidden-layer map
\[
G(u)=C\,\phi\!\big(\alpha(Bu+d)\big)+e,\qquad \text{(hidden width }W\ge S),
\]
such that $G(u_i)=u_i$ for all $i$.

\emph{Construction.}
Choose $S$ rows of $B$ and corresponding entries of $d$ so that the $S\times S$ matrix
\[
M^{(\alpha)}_{ik}:=\phi\!\big(\alpha(b_k^\top u_i + d_k)\big),\qquad 1\le i,k\le S,
\]
has the same lower-triangular limit as in Step~1; hence $M^{(\alpha^\star)}$ is invertible for some finite $\alpha^\star$.
Let $U:=[u_1\,\cdots\,u_S]\in\mathbb{R}^{m\times S}$ and define $C\in\mathbb{R}^{m\times W}$ by
\[
C=\big[\,U\,(M^{(\alpha^\star)})^{-1}\ \ \ \ 0_{m\times (W-S)}\,\big],
\]
i.e., the last $W-S$ columns are zeros. Set the last $W-S$ rows of $B$ and entries of $d$ arbitrarily, and take $e=0$.
Then for each $i$,
\[
G(u_i)=C\,\phi\!\big(\alpha^\star(Bu_i+d)\big)
      = U\,(M^{(\alpha^\star)})^{-1}\,M^{(\alpha^\star)}_{\cdot i}
      = u_i.
\]
Hence $G$ acts as the identity on $U$ while using width $W\ge S$ (extra units are harmlessly zeroed out).

\smallskip
\noindent\textbf{Step 3 (Depth lifting to any $L\ge2$ with width $W$).}
Let $h_i\in\mathbb{R}^{W}$ denote the hidden code of $x_i$ produced by the hidden layer of $N_2$ (the last $W-S$ coordinates may be arbitrary but are fixed).
Apply Step~2 to $U=\{h_1,\dots,h_S\}\subset\mathbb{R}^{W}$ to obtain a width-$W$ block $G$ with $G(h_i)=h_i$ for all $i$.
Insert $G$ between the hidden layer and the output of $N_2$; this does not change the outputs on $X$.
Repeating this insertion yields a depth-$L$ network of width $W$ that still satisfies $N(x_i)=f(x_i)$ for all $i$.

\smallskip
Combining the steps, for any $L\ge2$ and any $W\ge S$ there exists a depth-$L$, width-$W$ network with activation $\phi$ that interpolates $f$ exactly on $X$.

1) The construction plainly subsumes the case $W=S$ by taking the last $W-S$ (i.e., zero) columns to be vacuous.
2) For vector-valued $f:X\to\mathbb{R}^m$, share the same hidden layers and use an $m$-dimensional linear readout, or solve the linear systems coordinate-wise; the width requirement remains $W\ge S$.
\end{proof}

\subsection{Bounding the Statistical Error}
\label{appendix:stat_error}
Next, we bound the statistical error $\mathcal{E}_{\text{stat}}$. This error is due to estimating  $s^*_{\theta,k}(x_{k})$ from a finite number of samples rather than the full data distribution.

\textbf{Recall Lemma \ref{lemma:stat_error} (Statistical Error.)} \textit{Let $n_k$ denote the number of samples used to estimate the score function at each diffusion step. Then, for all $k \in [0,K-1]$, with probability at least $1-\gamma$, we have
\begin{equation}
\mathcal{E}^{\mathrm{stat}}_k \leq \mathcal{O}\left((S-1)d\sqrt\frac{{{\log}\left(\frac{2}{\gamma}\right)}}{{n_k}}\right)
\end{equation}}

\begin{proof}
\label{proof:statistical_eror} 

Let us define the population loss at time $k$ for $k \in [0,K-1]$ as follows:
\begin{equation}
\begin{aligned}
\mathcal{L}_k(\theta) &=
\int_{kh+\delta}^{(k+1)h+\delta}\mathbb{E}_{\mathbf{x}_k \sim q_t} \left[ 
D_I \left( 
s_k(\mathbf{x}_t) 
\,\|\, 
\overline{s}_{\theta, k}(\mathbf{x}_t) 
\right) 
\right]dt
\end{aligned}
\label{eqn:PRM_loss_1}
\end{equation}

The corresponding empirical loss is:

\begin{equation}
\widehat{\mathcal{L}}_k(\theta) = 
\frac{1}{n_k} \sum_{i=1}^{n_k} 
D_I\left(
s_k(\mathbf{x}_{k,t,i}) 
\,\|\, 
\overline{s}_{\theta, k}(\mathbf{x}_{k,t,i})
\right),\nonumber
\quad t \sim \text{Unif}[kh,(k+1)h]
\label{eqn:ERM_loss}
\end{equation}

Let $\theta^a_k$ and $\theta^b_k$ be the minimizers of $\mathcal{L}_k(\theta)$ and $\widehat{\mathcal{L}}_k(\theta)$, respectively, corresponding to score functions $s^a_{\theta, k}$ and $s^b_{\theta, k}$. By the definitions of minimizers, we can write
\begin{align}
    \mathcal{L}_k(\theta^b_k) - \mathcal{L}_k(\theta^a_k)
    &\le  \mathcal{L}_k(\theta^b_k) - \mathcal{L}_k(\theta^a_k) + \widehat{\mathcal{L}}_k(\theta^a_k) - \widehat{\mathcal{L}}_k(\theta^b_k)  \label{temp}
    \\
    &\le \underbrace{ \left| \mathcal{L}_k(\theta^b_k) - \widehat{\mathcal{L}}_k(\theta^b_k) \right| }_{\text{(I)}} + \underbrace{ \left| \mathcal{L}_k(\theta^a_k) - \widehat{\mathcal{L}}_k(\theta^a_k) \right| }_{\text{(II)}}.
    \label{eq:loss_gap}
\end{align}

Note that right hand side of (\ref{temp}) is less then the left hand side since we have added the quantity $ \widehat{\mathcal{L}}_k(\theta^a_k) - \widehat{\mathcal{L}}_k(\theta^b_k)$ which is strictly positive since $\theta^{b}_{k}$ is the minimizer of the function $\widehat{\mathcal{L}}_k(\theta)$ by definition. We then take the absolute value on both sides of (\ref{eq:loss_gap}) to get 

\begin{align}
    |\mathcal{L}_k(\theta^b_k) - \mathcal{L}_k(\theta^a_k)|
   &\le \underbrace{ \left| \mathcal{L}_k(\theta^b_k) - \widehat{\mathcal{L}}_t(\theta^b_k) \right| }_{\text{(I)}} + \underbrace{ \left| \mathcal{L}_k(\theta^a_k) - \widehat{\mathcal{L}}_k(\theta^a_k) \right| }_{\text{(II)}}.
    \label{eq:loss_gap_1}
\end{align}

Note that this preserves the direction of the inequality since the left-hand side of (\ref{eq:loss_gap}) is strictly positive since $\theta^{a}_{k}$ is the minimizer of $\mathcal{L}_k(\theta)$ by definition.

We now bound terms (I) and (II) using standard generalization results. From lemma \ref{lemma:thm26_5shalev}, if the loss function is uniformly bounded over the parameter space $\Theta'' = \{\theta^a_k, \theta^b_k\}$, then with probability at least $1 - \gamma$, we have
\begin{equation}
    \mathbb{E} \left[ \left| \mathcal{L}_k(\theta) - \widehat{\mathcal{L}}_k(\theta) \right| \right] \le \widehat{R}(\theta) + \mathcal{O}\left( \sqrt{ \frac{ \log \frac{1}{\gamma} }{n_k} } \right),~~\forall ~ \theta \in \Theta'' 
\end{equation}
where $\widehat{R}(\theta)$ denotes the empirical Rademacher complexity of the function class restricted to $\Theta''$.
Since $\Theta''$ is a finite class (just two functions) and the inputs are finite, the loss function $\mathcal{L}_k(\theta)$ is bounded, thus, We have from \ref{lemma:thm26_5shalev}
\begin{equation}
    \mathbb{E} \left[ \left| \mathcal{L}_k(\theta) - \widehat{\mathcal{L}}_k(\theta) \right| \right] \le \widehat{R}(\theta) + \mathcal{O}\left( \sqrt{ \frac{ \log \frac{1}{\gamma} }{n_k} } \right), ~~\quad \forall ~\theta \in \Theta''. \label{main_temp}
\end{equation}

From Lemma \ref{lemm:mass_ex} we have 

\begin{align}
    \mathbb{E}_{\sigma}\max_{\theta \in \Theta^{''}}\sum_{i=1}^{n_k}f(\theta){\sigma_{i}} \le ||f(\theta)||_{2} \le (BW)^L\!\left((S-1)d+\frac{L}{W}\right)
\end{align}

Now, the definition of $\hat{R}(\theta)$ is $\hat{R}(\theta)= \frac{1}{m}    \mathbb{E}_{\sigma}\max_{\theta \in \Theta^{''}}\sum_{i=1}^{n}f(\theta){\sigma_{i}}$. Thus we get:

\begin{equation}
    \mathbb{E} \left[ \left| \mathcal{L}_k(\theta) - \widehat{\mathcal{L}}_k(\theta) \right| \right] \le \mathcal{O}\left(\frac{(W)^L\left((S-1)d+\frac{L}{W}\right)}{n_k}\right) + \mathcal{O}\left( \sqrt{ \frac{ \log \frac{1}{\gamma} }{n_k} } \right), ~~\quad \forall ~\theta \in \Theta''. \label{r2}
\end{equation}

Here $||\Theta||$ is a value such that $||\Theta|| = \sup_{\theta \in \Theta^{''}}||\theta||_{\infty}$. Which implies that $\Theta$ is the norm of the maximum parameter value that can be attained in the finite function class $\Theta^{''}$,   which yields:

\begin{equation}
    \mathbb{E} \left[ \left| \mathcal{L}_k(\theta) - \widehat{\mathcal{L}}_k(\theta) \right| \right] \le \mathcal{O} \left((W)^L\left((S-1)d+\frac{L}{W}\right) \sqrt{ \frac{ \log \frac{1}{\gamma} }{n_k} } \right), ~~\quad \forall ~\theta \in \Theta''
\end{equation}
Finally, using the Polyak-Łojasiewicz (PL) condition for $\mathcal{L}_k(\theta)$, from Assumption \ref{ass:PL-condition}, we have from the quadratic growth condition of PL functions the following,  
\begin{equation}
    \| \theta^a_k - \theta^b_k \|^2 \le \mu \left| \mathcal{L}_k(\theta^a_k) - \mathcal{L}_k(\theta^b_k) \right|,
\end{equation}
Define the constant $L'$ as $L'$ =  $\text{sup}_{L'_{k}}\| s^a_{\theta,k}(x_k) - s^b_{\theta,k}(x_k)\|^2 \le L'_k.\| \theta^a_k - \theta^b_k \|^2$. Then we have
\begin{align}
    \| s^a_{\theta,k}(x_k) - s^b_{\theta,k}(x_k)\|^2 &\le L'\| \theta^a_k - \theta^b_k \|^2 \\
    &\le L'\mu \left| \mathcal{L}_k(\theta^a_k) - \mathcal{L}_k(\theta^b_k) \right| \\
    &\le \mathcal{O} \left(  (W)^L\left((S-1)d+\frac{L}{W}\right) \sqrt{ \frac{\log \frac{1}{\gamma} }{n_k} } \right)
\end{align}
Taking expectation with respect to $x \sim q_{k}$ on both sides completes the proof. 
\end{proof}

\subsection{Bounding the Optimization Error}\label{appendix:opt_error}
The optimization error ($\mathcal{E}_{\text{opt}}$) accounts for the fact that gradient-based optimization does not necessarily find the optimal parameters due to limited steps, local minima, or suboptimal learning rates. This can be bounded as follows.

\textbf{Recall Lemma \ref{lemma:opt_error} (Optimization Error)} Let $ n_{k}$ be the number of samples used to estimate the score function at diffusion step $k$. If the learning rate $0 \le \eta \le \frac{1}{\kappa}$, then under Assumptions \ref{ass:PL-condition} and \ref{ass:smoothness}, for all $ k \in [0,K-1] $, the optimization error due to imperfect minimization of the training loss satisfies with probability at least $1-\gamma$
\begin{equation}
    \mathcal{E}^{\mathrm{opt}}_k \leq \mathcal{O} \left((S-1)d\sqrt{\frac{ \log \left(\frac{2}{\gamma} \right)}{n_k}} \right).
\end{equation}

\begin{proof}
    Let $\mathcal{E}_k^{\text{opt}}$ denote the optimization error incurred when performing stochastic gradient descent (SGD), with the empirical loss defined by: 

\begin{equation}
\widehat{\mathcal{L}}_k^i(\theta) = 
D_I\left(
s_k(\mathbf{x}_{k,t,i}) 
\,\|\, 
\overline{s}_{\theta,k}(\mathbf{x}_{k,i})
\right)
\quad t \sim \text{Unif}[kh,(k+1)h]
\label{eqn:ERM_loss_1}
\end{equation}

The corresponding population loss is:
\begin{equation}
\begin{aligned}
\mathcal{L}_k(\theta) 
&= \int_{kh+\delta}^{(k+1)h+\delta}\mathbb{E}_{\mathbf{x}_k \sim q_t} \left[ 
D_I \left( 
s_k(\mathbf{x}_t) 
\,\|\, 
\overline{s}_{\theta, k}(\mathbf{x}_t) 
\right) 
\right]dt
\end{aligned}
\label{eqn:PRM_loss}
\end{equation}

Thus, $\mathcal{E}_k^{\text{opt}}$ captures the error incurred during the stochastic optimization at each fixed time step $k$. We now derive upper bounds on this error.

From the smoothness of $ \mathcal{L}_k(\theta) $ through Assumption \ref{ass:smoothness}, we have
\begin{align}
    \mathcal{L}_k(\theta_{i+1}) \le \mathcal{L}_{k}(\theta_i) + \langle \nabla \mathcal{L}_k(\theta_i), \theta_{i+1} - \theta_i \rangle + \frac{\kappa}{2} \|\theta_{i+1} - \theta_i\|^2.
\end{align}

Taking conditional expectation given $ \theta_i $, and using the unbiased-ness of the stochastic gradient $ \nabla \widehat{\mathcal{L}}_k(\theta_i) $, we get:

\begin{align}
\mathbb{E}[\mathcal{L}_k(\theta_{i+1}) \mid \theta_i] 
&\le \mathcal{L}_k(\theta_i) 
- \alpha_t \left\| \nabla \mathcal{L}_k(\theta_i) \right\|^2 \nonumber \\
&\quad + \frac{\kappa \alpha_t^2}{2} \, 
\mathbb{E}\left[ \left\| \nabla \widehat{\mathcal{L}}_k(\theta_i) \right\|^2 \mid \theta_i \right]
\end{align}

Now using the variance bound on the stochastic gradients using Assumption \ref{ass:smoothness}, we have
\begin{align}
    \mathbb{E}[\|\nabla \widehat{\mathcal{L}}_k(\theta_i)\|^2 \mid \theta_i] \le \|\nabla \mathcal{L}_{k}(\theta_i)\|^2 + \sigma^2,
\end{align}
Using this in the previous equation, we have that

\begin{align}
\mathbb{E}[\mathcal{L}_{k}(\theta_{t+1}) \mid \theta_i] 
&\le \mathcal{L}_{k}(\theta_i) 
- \eta \left\| \nabla \mathcal{L}(\theta_i) \right\|^2 
+ \frac{\kappa \eta^2}{2} 
\left( \left\| \nabla \mathcal{L}(\theta_i) \right\|^2 + \sigma^2 \right) \nonumber \\
&= \mathcal{L}(\theta_i) 
- \left( \eta - \frac{\kappa \eta^2}{2} \right) 
\left\| \nabla \mathcal{L}(\theta_i) \right\|^2 
+ \frac{\kappa \eta^2 \sigma^2}{2}
\end{align}

Now applying the PL inequality (Assumption \ref{ass:PL-condition}), $\|\nabla L(\theta_i)\|^2 \ge 2\mu \left(L(\theta_i) - L^* \right)$, we substitute in the above inequality to get
\begin{align}
    \mathbb{E}[\mathcal{L}(\theta_{i+1}) \mid \theta_i] - \mathcal{L}^* 
    &\le \left(1 - 2\mu \left(\eta - \frac{\kappa \eta^2}{2} \right)\right)\left(\mathcal{L}(\theta_i) - \mathcal{L}^*\right) + \frac{\kappa \eta^2 \sigma^2}{2}.
\end{align}

Define the contraction factor
\begin{equation}
\rho := 1 - 2\mu \left( \eta - \frac{\kappa \eta^2}{2} \right)
\end{equation}

Taking total expectation and defining $ \delta_k = \mathbb{E}[L(\theta_i) - L^*] $, we get the recursion:
\begin{align}
\delta_{k+1} \le \rho \cdot \delta_k + \frac{\kappa \eta^2 \sigma^2}{2}.
\end{align}
When $ \eta \le \frac{1}{\kappa} $, we have
\begin{align}
\eta - \frac{\kappa \eta^2}{2} \ge \frac{\eta}{2} \Rightarrow \rho \le 1 - \mu \eta.
\end{align}
Unrolling the recursion we have
\begin{align}
\delta_k \le (1 - \mu \eta)^k \delta_0 + \frac{\kappa \eta^2 \sigma^2}{2} \sum_{j=0}^{t-1} (1 - \mu \eta)^j.
\end{align}

Using the geometric series bound:
\begin{align}
    \sum_{j=0}^{k-1} (1 - \mu \eta)^j \le \frac{1}{\mu \eta},
\end{align}
we conclude that
\begin{align}
    \delta_k \le (1 - \mu \eta)^k \delta_0 + \frac{\kappa \eta \sigma^2}{2\mu}.
\end{align}
Hence, we have the convergence result
\begin{align}
    \mathbb{E}[\mathcal{L}(\theta_i) - \mathcal{L}^*] \le (1 - \mu \eta)^k \delta_0 + \frac{\kappa \eta \sigma^2}{2\mu}.
\end{align}

From theorem 4.6 of \cite{gower2021sgd} states that this result implies a sample complexity of $\mathcal{O}\left(\frac{1}{n}\right)$ where $n$ is the number of steps of the SGD algorithm performed. Note that the  result from \cite{gower2021sgd} requires bounded gradient and smoothness assumption which are satisfied by Assumption \ref{ass:smoothness}.

Note that $\overline{s}_{k}$ and $\overline{\theta}_{k}$ denote our estimate of the score function and associated parameter obtained from the SGD.
Also note that $\mathcal{L}^{*}$ is the loss function corresponding to whose minimizer is the neural network $s_{k}^{a}$ and the neural parameter $\theta_{k}^{a}$ is our estimated score parameter. Thus, applying the quadratic growth inequality.
\begin{align}
\left\| \overline{s}_k(\mathbf{x}_k) - s_k^a(\mathbf{x}_k) \right\|^2 
&\le L \left\| \overline{\theta}_{k} - \theta_k^{a} \right\|^2 
\le \left| \mathcal{L}_k(\theta_k) - \mathcal{L}_k^* \right| \\
&\le \mathcal{O}\left( \frac{1}{n} \right)
\end{align}

From Lemma~\ref{lemma:stat_error}, we have with probability \(1 - \gamma\) that
\begin{align}
\left\| s_k^a(\mathbf{x}_k) - s_k^b(\mathbf{x}_k) \right\|^2 
&\le L  \left\| \theta^a_k - \theta^b_k \right\|^2 \\
&\le L  \mu \left| \mathcal{L}_k(\theta^a_k) - \mathcal{L}_k(\theta^b_k) \right| \\
&\le \mathcal{O} \left( (W)^L\left((S-1)d+\frac{L}{W}\right) \sqrt{ \frac{ \log \frac{2}{\gamma} }{n_k} } \right)
\end{align}

Thus, we have with probability at least \(1 - \gamma\),
\begin{align}
\left\| \overline{s}_{\theta,k}(\mathbf{x}_k) - s^b_{\theta,k}(\mathbf{x}_k) \right\|^2 
&\le 2 \left\| \overline{s}_{\theta,k}(\mathbf{x}_k) - s_{\theta,k}^a(\mathbf{x}_k) \right\|^2 
+ 2 \left\| s^a_{\theta,k}(\mathbf{x}_k) - s^b_{\theta,k}(\mathbf{x}_k) \right\|^2 \nonumber \\
&\le \mathcal{O}\left( \log\left( \frac{1}{n} \right) \right) 
+ \mathcal{O}\left( (W)^L\left((S-1)d+\frac{L}{W}\right)  \sqrt{ \frac{ \log \frac{2}{\gamma} }{n_k} } \right) \nonumber \\
&\le \mathcal{O} \left((W)^L\left((S-1)d+\frac{L}{W}\right)  \sqrt{ \frac{ \log \frac{2}{\gamma} }{n_k} } \right)
\end{align}

Taking expectation with respect to 
\( \mathbf{x}_k \sim q_k \) on both sides completes the proof.
\end{proof}

\subsection{Bounding the Clipping Error} \label{appendix:clip_error}
The clipping error ($\mathcal{E}_{\text{clip}}$) accounts for the deviation of the unclipped score network $\overline{s}_{\theta, k}(.)$ from the interval \([a,b]^{(S-1)d}\). This can be bounded as follows. 



\begin{lemma}[Clipping Error]
\label{lemma:clip_error}
   Let $\mathcal{X} = [S]^d$ and for each $k \in [0, K-1]$ let the unclipped score be
$\;\overline{s}_{\theta,k}(\cdot):\mathcal{X} \to \mathbb{R}^{d(S-1)}\,$. Let the true score
$s_{k}(\cdot):\mathcal{X} \to [1/C, C]^{d(S-1)}$ and the clipped score $\hat{s}_{\theta, k}(\cdot)$ be component-wise bounded in $[1/C, C]$. The clipped score $\hat{s}_{\theta, k}(\cdot)$ can be component-wise restricted to $[1/C, C]$ as follows:
\begin{equation}
\hat{s}_{\theta,k}(x_k)_{i,c}
\;=\;
\min\!\big\{C,\;\max\{1/C,\;\bar{s}_{\theta,k}(x_k)_{i,c}\}\big\},
i\in[d],\; c\in[S]\!\setminus\!\{x_k^i\}.
\end{equation}
and hence, for any distribution $q_k$ on $\mathcal{X}$,
\begin{equation}\label{eq:expected-contraction}
\mathbb{E}_{x_k\sim q_k}\!\bigl[\|\hat{s}_{\theta,k}(x_k)-\bar{s}_{\theta, k}(x_k)\|_2^2\bigr] 
\;\le\;
\mathbb{E}_{x_k\sim q_k}\!\bigl[\|\overline{s}_{\theta,k}(x_k)-s_{k}(x_k)\|_2^2\bigr].
\end{equation}
\end{lemma}

\begin{proof}
Partition the $d(S-1)$ component indices $(i,c)$ of the score vector according to the \emph{unclipped} value:
\begin{align}
\kappa_1(x_k) &= \{(i,c): \bar{s}_{\theta,k}(x_k)_{i,c} < a\},\qquad
\kappa_2(x_k) = \{(i,c): a \le \bar{s}_{\theta,k}(x_k)_{i,c} \le b\},\nonumber\\
\kappa_3(x_k) &= \{(i,c): \bar{s}_{\theta,k}(x_k)_{i,c} > b\}.
\end{align}
By construction of the projection onto $[a,b]$, on $\kappa_1(x_k)$ we have
$\hat{s}_{\theta,k}(x_k)_{i,c}=a$; on $\kappa_2(x_k)$, $\hat{s}_{\theta,k}(x_k)_{i,c}=\bar{s}_{\theta,k}(x_k)_{i,c}$; and on $\kappa_3(x_k)$, $\hat{s}_{\theta,k}(x_k)_{i,c}=b$.
Therefore, the squared error decomposes as
\begin{align}
\|\hat{s}_{\theta,k}(x_k)-\bar{s}_{\theta, k}(x_k)\|_2^2
&= \sum_{(i,c)\in\kappa_1(x_k)} \bigl(a - \bar{s}_{\theta, k}(x_k)_{i,c}\bigr)^2
 + \sum_{(i,c)\in\kappa_2(x_k)} \bigl(\hat{s}_{\theta,k}(x_k)_{i,c} - \bar{s}_{\theta, k}(x_k)_{i,c}\bigr)^2 \nonumber\\
&\hspace{3em}
 + \sum_{(i,c)\in\kappa_3(x_k)} \bigl(b - \bar{s}_{\theta,k}(x_k)_{i,c}\bigr)^2.
\label{eq:decomp}
\end{align}

In Equation (\ref{eq:decomp}), $\sum_{(i,c)\in\kappa_2(x_k)} \bigl(\hat{s}_{\theta,k}(x_k)_{i,c} - \bar{s}_{\theta, k}(x_k)_{i,c}\bigr)^2 = 0$ as $\hat{s}_{\theta,k}(x_k)_{i,c}=\bar{s}_{\theta,k}(x_k)_{i,c}$ on $\kappa_2(x_k)$. 
Since $s_{k}(x_k)_{i,c}\in[a,b]$ component-wise, the projection onto $[a,b]$ is
non-expansive in each coordinate. Concretely, for $(i,c)\in\kappa_1(x_k)$,
\begin{equation}
\bigl(a - \bar{s}_{\theta,k}(x_k)_{i,c}\bigr)^2 \;\le\; \bigl(\bar{s}_{\theta,k}(x_k)_{i,c} - s_{k}(x_k)_{i,c}\bigr)^2,
\label{eq:left-ineq}
\end{equation}
for $(i,c)\in\kappa_3(x_k)$,
\begin{equation}
\bigl(b - \bar{s}_{\theta,k}(x_k)_{i,c}\bigr)^2 \;\le\; \bigl(\bar{s}_{\theta,k}(x_k)_{i,c} - s_{k}(x_k)_{i,c}\bigr)^2,
\label{eq:right-ineq}
\end{equation}
and for $(i,c)\in\kappa_2(x_k)$ we have equality
\begin{equation}
\bigl(\hat{s}_{\theta,k}(x_k)_{i,c} - \bar{s}_{\theta,k}(x_k)_{i,c}\bigr)^2
= 0.
\label{eq:nocap}
\end{equation}
Summing (\ref{eq:left-ineq})-(\ref{eq:nocap}) over the three index sets and using
Equation (\ref{eq:decomp}) followed by taking expectation over $x\sim q_k$ proves (\ref{eq:expected-contraction}).
\end{proof}

\newcommand{\TV}{\mathrm{TV}}

\section{Proof of the Main Theorem} 
\label{app: appendix_main_theorem_full_support}
\subsection{Path measure KL decomposition}
The following decomposition follows from the Proof of Theorem 2 from \citet{zhang2025convergencescorebaseddiscretediffusion}. 
\begin{align}
D_{\mathrm{KL}}(\mathbb{Q}\,\|\,\mathbb{P}^{q_T})
&= \frac{1}{S} \sum_{k=0}^{K-1} \int_{kh}^{(k+1)h} 
\mathbb{E}_{\mathbf{x}_t \sim q_t} 
D_I\!\left(s_t(\mathbf{x}_t)\,\|\,\hat{s}_{(k+1)h}(\mathbf{x}_t)\right) dt \nonumber \\
&\;\overset{(i)}{\lesssim}\; \frac{C}{S} \sum_{k=0}^{K-1} \int_{kh}^{(k+1)h} 
\mathbb{E}_{\mathbf{x}_t \sim q_t} 
\bigl\| s_t(\mathbf{x}_t) - s_{(k+1)h}(\mathbf{x}_t) \bigr\|_2^2 \, dt \nonumber \\
&\quad + \frac{C^2}{S} \sum_{k=0}^{K-1} \int_{kh}^{(k+1)h} 
\mathbb{E}_{\mathbf{x}_t \sim q_t} 
D_I\!\left(s_{(k+1)h}(\mathbf{x}_t)\,\|\,\hat{s}_{(k+1)h}(\mathbf{x}_t)\right) dt \nonumber \\
&\;\overset{(ii)}{\lesssim}\; \frac{C}{S} \sum_{k=0}^{K-1} \int_{kh}^{(k+1)h} 
\mathbb{E}_{\mathbf{x}_t \sim q_t} \sum_{i=1}^d \sum_{\hat{x}^i \neq x^i} 
\bigl| s_t(\mathbf{x}_t)_{i,\hat{x}^i} - s_{(k+1)h}(\mathbf{x}_t)_{i,\hat{x}^i} \bigr|^2 dt \nonumber \\
&\quad + \frac{C^3}{S} \sum_{k=0}^{K-1} \int_{kh+\delta}^{(k+1)h+\delta} 
\mathbb{E}_{\mathbf{x}_t \sim q_t} 
\bigl\| s_{(k+1)h+\delta}(\mathbf{x}_t) - \hat{s}_{\theta,(k+1)h+\delta}(\mathbf{x}_t) \bigr\|_2^2 \, dt \nonumber \\
&\;\overset{(iii)}{\lesssim}\; \frac{C}{S} \sum_{k=0}^{K-1} \int_{kh}^{(k+1)h} 
\mathbb{E}_{\mathbf{x}_t \sim q_t} \sum_{i=1}^d \sum_{\hat{x}^i \neq x^i} 
\left[ \frac{1}{1 - e^{-(k+1)h}} + S \right]^2 \kappa_i^2 h^2 \, dt \nonumber \\
&\quad + \frac{C^3}{S} \sum_{k=0}^{K-1} h \,
\mathbb{E}_{\mathbf{x}_{(k+1)h+\delta} \sim q_{(k+1)h+\delta}}
\bigl\| s_{(k+1)h+\delta}(\mathbf{x}_{(k+1)h+\delta})
- \hat{s}_{\theta,(k+1)h+\delta}(\mathbf{x}_{(k+1)h+\delta}) \bigr\|_2^2 \nonumber \\
&\quad + \frac{C(S-1)d\lambda T h}{S} \nonumber \\
&\;\overset{(iv)}{\lesssim}\; C \sum_{k=0}^{K-1} \sum_{i=1}^d 
\left[ \frac{1}{(1-e^{-(k+1)h})^2} + S^2 \right] \kappa_i^2 h^3 \nonumber \\
&\quad + \frac{C^3}{S} \sum_{k=0}^{K-1} h \, A_k
+ \frac{C(S-1)d\lambda T h}{S} \nonumber \\
&\;\overset{(v)}{\lesssim}\; C \kappa^2 h^3 \int_h^T \frac{1}{(1-e^{-x})^2} dx 
+ C \kappa^2 S^2 h^2 T \nonumber \\
&\quad + \frac{C^3}{S} \sum_{k=0}^{K-1} h \, A_k
+ \frac{C(S-1)d\lambda T h}{S} \nonumber \\
&\;\overset{(vi)}{\lesssim}\; C \kappa^2 h^3 \bigl[ T - \log h + \tfrac{1}{h} \bigr] 
+ C \kappa^2 S^2 h^2 T \nonumber \\
&\quad + \frac{C^3}{S} \sum_{k=0}^{K-1} h \, A_k
+ \frac{C(S-1)d\lambda T h}{S} \nonumber \\
&\;\overset{(vii)}{\lesssim}\; C \kappa^2 h^2 T (h + S^2) \nonumber \\
&\quad + \frac{C^3}{S} \sum_{k=0}^{K-1} h \, A_k
+ \frac{C(S-1)d\lambda T h}{S} \nonumber \\
&\;\overset{(viii)}{\lesssim}\; C \kappa^2 S^2 h^2 T \nonumber \\
&\quad + \frac{C^3}{S} \sum_{k=0}^{K-1} h \, A_k
+ \frac{C(S-1)d\lambda T h}{S}.
\label{eqn: KL_decomp_with_full_support}
\end{align}

The first equality follows from Lemma \ref{lem:kl-divergence_path_measure}. Lemma \ref{lemma: bregman_decomposition} $(i)$ and \ref{lemma: score_uniform_bound} lead to the first inequality $(i)$. We particularly avoid the assumption $\frac{1}{S} \sum_{k=0}^{K-1} \int_{kh}^{(k+1)h} \mathbb{E}_{\mathbf{x}_t \sim q_t}\, D_I\!\left(s_{(k+1)h}(\mathbf{x}_t)\,\|\,\hat{s}_{\theta,(k+1)h}(\mathbf{x}_t)\right) dt \le \varepsilon_{\text{score}}$ from \citet{zhang2025convergencescorebaseddiscretediffusion} and further upper bound it by squared L2 norm, leading to the second inequality $(ii)$. The third inequality $(iii)$ arises from Lemma \ref{lemma: score_movement_bound_full_support} and Lemma \ref{thm:disc-vs-cont-abs}. The fourth inequality $(iv)$ stems from the fact that $(a+b)^2 \leq 2(a^2+b^2)$. 
The sixth inequality $(vi)$ is a consequence of $e^x - 1 \geq x$, and the final inequality $(viii)$ is due to the assumption that $h \leq S^2$.

\subsection{KL divergence bound}
From \citet{zhang2025convergencescorebaseddiscretediffusion} we use the result that $D_{\mathrm{KL}}(q_T \,\|\, \pi^d) = de^{-T}\log S$. Hence, by writing $M = C(S-1)d$ the final bound on $D_{KL}(p_{data} \,\|\, p_{T})$ from (\ref{eqn:KL_inequality}) and (\ref{eqn: KL_decomp_with_full_support}):
\begin{equation}
D_{KL}(p_{data} \,\|\, p_{T})
\;\lesssim\; de^{-T}\log S + C \kappa^2 S^2 h^2 T 
+ \frac{C^3}{S} \sum_{k=0}^{K-1} h A_k + \frac{M\lambda T h}{S}.
\label{eqn:final-kl-bound}
\end{equation}

We shall now focus on bounding the third term in (\ref{eqn:final-kl-bound}) which expresses the score approximation error by writing $A_k = \mathbb{E}_{\mathbf{x}_{(k+1)h} \sim q_{(k+1)h}} \left\| s_{(k+1)h}(\mathbf{x}_{(k+1)h}) - \hat{s}_{\theta,(k+1)h}(\mathbf{x}_{(k+1)h}) \right\|_2^2$. For notational simplicity we re-write $A_k = \mathbb{E}_{\mathbf{x}_k \sim q_k} 
\left\| s_k(\mathbf{x}_k) \right. 
- \left. \hat{s}_{\theta,k}(\mathbf{x}_k) \right\|_2^2$ as mentioned above (\ref{eq:Ak_decomp}). For $D_{\mathrm{KL}}(q_{0} \| p_{T}) \leq \tilde{O}(\epsilon)$ to hold, we have to ensure that each of the terms from (\ref{eqn:final-kl-bound}) scales as $\tilde{O}(\epsilon)$. Therefore, we bound the score-approximation error term $\frac{C^3}{S} \sum_{k=0}^{K-1} h \,
 \mathbb{E}_{\mathbf{x}_{(k+1)h} \sim q_{(k+1)h}}
\bigl\| s_{(k+1)h}(\mathbf{x}_{(k+1)h})
 - \hat{s}_{\theta,(k+1)h}(\mathbf{x}_{(k+1)h}) \bigr\|_2^2 $ based on the number of samples $n_k$ required to achieve an $\epsilon_{score}$-accurate score estimator.
 
\begin{equation}
D_{KL}(q \,\|\, p_{T})
\;\lesssim\; de^{-T}\log S \;+\; C \kappa^2 S^2 h^2 T
\;+\; \frac{C^3}{S}\sum_{k=0}^{K-1} h A_k \;+\; \frac{M\lambda T h}{S}
\end{equation}

\begin{equation}\label{eq:T-choice}
de^{-T}\log S \;\le\; \epsilon
\quad\Longrightarrow\quad
T \;\gtrsim\; \log\!\Big(\tfrac{d\log S}{\epsilon}\Big)
\end{equation}

\begin{equation}\label{eq:h-bounds-each}
C \kappa^2 S^2 h^2 T \;\le\; \epsilon
\;\;\Longrightarrow\;\;
h \;\lesssim\; \Big(\tfrac{\epsilon}{C \kappa^2 S^2 T}\Big)^{1/2},
\qquad
\frac{M\lambda T h}{S} \;\le\; \epsilon
\;\;\Longrightarrow\;\;
h \;\lesssim\; \tfrac{\epsilon S}{M\lambda T}
\end{equation}

\begin{equation}\label{eq:h-min}
h \;\lesssim\;
\min\!\left\{
\Big(\tfrac{\epsilon}{C \kappa^2 S^2 T}\Big)^{1/2},
\;\tfrac{\epsilon S}{M\lambda T}
\right\}
\end{equation}

\begin{equation}\label{eq:K-bound}
K \;=\; \frac{T}{h}
\;\gtrsim\;
\max\!\left\{
T\,\Big(\tfrac{C \kappa^2 S^2 T}{\epsilon}\Big)^{1/2},
\;\tfrac{M\lambda}{\epsilon}\,\tfrac{T^2}{S}
\right\}
\end{equation}

\begin{equation}\label{eq:K-explicit}
T \asymp \log\!\Big(\tfrac{d\log S}{\epsilon}\Big)
\;\Longrightarrow\;
K \;\gtrsim\;
\max\!\left\{
\sqrt{\tfrac{C \kappa^2 S^2}{\epsilon}}\,
\big[\log(\tfrac{d\log S}{\epsilon})\big]^{3/2},
\;\tfrac{M\lambda}{\epsilon S}\,
\big[\log(\tfrac{d\log S}{\epsilon})\big]^2
\right\}
\end{equation}

\begin{equation}
D_{\mathrm{KL}}(p_{\text{data}} \,\|\, p_T)
\;\le\; \tilde{O}\!\left(
\epsilon \;+\; \frac{C^{3}}{S}\sum_{k=0}^{K-1} h\,A_k
\right)
\end{equation}

Using the following bounds from earlier proofs: Bound on $\mathcal{E}_k^{\mathrm{stat}}$ from Lemma \ref{lemma:stat_error} and $\mathcal{E}_k^{\mathrm{opt}}$ from Lemma \ref{lemma:opt_error} as follows, \\
$\mathcal{E}_k^{\mathrm{stat}} \leq \mathcal{O} \left(  (W)^L\left((S-1)d+\frac{L}{W}\right) \sqrt{ \frac{\log \frac{1}{\gamma} }{n_k} } \right)$ and $\mathcal{E}_k^{\mathrm{opt}} \leq \mathcal{O} \left((W)^L\left((S-1)d +\frac{L}{W}\right) \cdot \sqrt{ \frac{ \log \frac{2}{\gamma} }{n_k} } \right)$\\
we obtain the following bound on $A_k$ which holds with a probability of at least $1-\gamma$: 

\begin{align}
A_k 
&\le 
\underbrace{\mathcal{O} \left(  (W)^L\left((S-1)d +\frac{L}{W}\right) \sqrt{ \frac{\log \frac{1}{\gamma} }{n_k} } \right)
}_{\text{Statistical Error}}
+ 
\underbrace{
\mathcal{O} \left((W)^L\left((S-1)d +\frac{L}{W}\right) \cdot \sqrt{ \frac{ \log \frac{2}{\gamma} }{n_k} } \right)
}_{\text{Optimization Error}}
\\
&\le 
\mathcal{O} \left((W)^L\left((S-1)d +\frac{L}{W}\right) \cdot \sqrt{ \frac{ \log \frac{2}{\gamma} }{n_k} } \right)
\label{eqn: Ak_bound}
\end{align}

We obtain (\ref{eqn: Ak_bound}) by appropriately combining the statistical and optimization error terms. We have to ensure that $A_k$ satisfies the bound in (\ref{eqn: Ak_bound}) by setting the parameters $n_k$ (sample size at discrete time-step $k$). Simultaneously, $A_k$ must also satisfy the following inequality to ensure that the score approximation error is less than $\epsilon_{score}$, which is the desired level of accuracy that we want to achieve, and characterize it with the number of samples and the neural network parameters required. We thus shift the analysis from an assumption of access to $\epsilon_{score}$-accurate score estimator to a learning-theoretic view by characterizing the number of samples $n_k$ and network parameters $WD$ required to achieve it. 

\begin{equation}
\frac{1}{S} \sum_{k=0}^{K-1} h \, A_k \leq \epsilon_{\mathrm{score}}.
\label{eqn: total_ak}
\end{equation}
where $\epsilon_{score}$ is the desired accuracy of the score estimator.

\begin{equation}
A_k \leq \frac{S}{Kh} \, \epsilon_{\mathrm{score}}.
\label{eqn: Ak_bound_eps}
\end{equation}

which we obtain as the worst-case bound on the per-discrete time-step squared L2 error. Here, we specifically note that the bound on $A_k$ does not depend on $k$ since the number of samples required to train at the $k^{th}$ time-step is available across all time-steps i.e., $n_k = n$.\\
\\
We isolate the two sources of error: \textbf{statistical and optimization}. We will solve for the sample size \( n_k = n \).\\

So the bound on $A_k$ becomes:
\begin{equation}\label{eq:Ak}
A_k \;\le\; \frac{S}{Kh}\,\epsilon_{\text{score}}
\;=\; O\!\left(\frac{S}{Kh\,C^3}\right).
\end{equation}

To ensure the above inequality holds, we enforce:
\begin{equation}\label{eq:ineq}
W^{\,L}\left((S-1)d+\frac{L}{W}\right)\sqrt{\frac{\log(2/\gamma)}{n_k}}
\;\le\; \frac{S}{T\,C^{3}}\,\epsilon.
\end{equation}

\begin{equation}\label{eq:nk-sol}
n_k \;\gtrsim\; 
\frac{C^{6}\,T^{2}}{S^{2}\epsilon^{2}}\;
W^{\,2L}\!\left((S-1)d+\frac{L}{W}\right)^{2}\;
\log\!\Big(\tfrac{2}{\gamma}\Big)
\end{equation}

In (\ref{eq:nk-sol}), recall that $Kh = T$, and $T$ can be set as 
$T \simeq \log\!\big(\tfrac{d\log S}{\epsilon}\big)$ from (\ref{eq:T-choice}). 
\begin{equation}
n_k \;\gtrsim\;
\frac{C^{6}}{S^{2}\epsilon^{2}}\,
W^{2L}\!\left(\,(S-1)d+\frac{L}{W}\right)^{2}
\Bigl[\log\!\Bigl(\tfrac{d\log S}{\epsilon}\Bigr)\Bigr]^{2}\,
\log\!\Bigl(\tfrac{2}{\gamma}\Bigr).
\label{eqn:nk_bound_logs}
\end{equation}



Hiding the logarithmic factors, we obtain the following bound:
\begin{equation}\label{eq:nk-final-tilde}
n_k \;=\; \tilde{\Omega}\!\left(
\frac{C^{6}}{S^{2}\epsilon^{2}}\;
W^{\,2L}\!\left((S-1)d+\frac{L}{W}\right)^{2}
\right)
\end{equation}





This completes the proof of the main theorem.

\section{Hardness of Learning in Discrete-State Diffusion Models} \label{hardness}

\begin{lemma}[Hardness of Learning in Discrete-State Diffusion Models]
\label{lemma:kl_counterexample}
The generation of diffusion models with KL gap $<\epsilon$ needs at least $\Theta(1/\epsilon^2)$ samples for some distribution.

\end{lemma}

To show this result, we will provide a construction which provides a counterexample showing that one cannot learn discrete distributions within $\varepsilon$-KL error using fewer than $\Theta(1/\varepsilon^2)$ samples. Otherwise, such a learner would violate the classical mean-estimation lower bound.

\textbf{Setup.}
Let $x_0,x_1\in\mathbb{R}$ with $x_0\neq x_1$, and $0<\varepsilon<\tfrac{1}{25}$. 
Define two-point laws
\begin{equation}
P=(1-\varepsilon)\,\delta_{x_0}+\varepsilon\,\delta_{x_1},
\qquad
Q=(1-25\,\varepsilon)\,\delta_{x_0}+25\,\varepsilon\,\delta_{x_1}.
\label{eq:two_point_PQ}
\end{equation}
Then
\begin{equation}
\mathbb{E}_P[X]=x_0+\varepsilon\,(x_1-x_0),\qquad
\mathbb{E}_Q[X]=x_0+25\,\varepsilon\,(x_1-x_0),
\end{equation}
\begin{equation}
\mathrm{Var}_P(X)=\varepsilon(1-\varepsilon)\,(x_1-x_0)^2,\qquad
\mathrm{Var}_Q(X)=(25\,\varepsilon)(1-25\,\varepsilon)\,(x_1-x_0)^2,
\end{equation}
so both means and variances are $\Theta(\varepsilon)$ when $|x_1-x_0|=\Theta(1)$.

For the two-point laws,
\[
D_{\mathrm{KL}}(P\|Q)
=(1-\varepsilon)\log\frac{1-\varepsilon}{1-25\varepsilon}
+\varepsilon\log\frac{\varepsilon}{25\varepsilon}.
\]
Using $\log(1-t)=-t-\tfrac{t^2}{2}+O(t^3)$ and $\log\frac{\varepsilon}{25\varepsilon}=\log\frac{1}{25}$, we get
\[
D_{\mathrm{KL}}(P\|Q)
=\bigl(24-\log 25\bigr)\varepsilon + O(\varepsilon^{2})
=\,c_1\,\varepsilon+O(\varepsilon^{2}),
\]
with $c_1\approx 20.78$.
As $\varepsilon\to0$, the leading term dominates, and hence
$D_{\mathrm{KL}}(P\|Q)=\Theta(\varepsilon)$.

\textbf{Hellinger-based KL separation.}
For two-point laws supported on $\{x_0,x_1\}$ with masses $(1-p,p)$ and $(1-q,q)$, the squared Hellinger distance is
\begin{equation}
H^2\big((1-p,p),(1-q,q)\big)=1-\Big(\sqrt{(1-p)(1-q)}+\sqrt{pq}\Big).
\end{equation}
Substituting $p=\varepsilon$ and $q=25\,\varepsilon$ gives
\begin{equation}
H^2(P,Q)=1-\sqrt{(1-\varepsilon)(1-25\,\varepsilon)}-\sqrt{25}\,\varepsilon.
\end{equation}
Using $\sqrt{1-t}\le 1-\tfrac{t}{2}$ for $t\in[0,1]$ with $t=26\varepsilon-25\varepsilon^2$, we have
\begin{equation}
\sqrt{(1-\varepsilon)(1-25\,\varepsilon)}\le 1-13\,\varepsilon+12.5\,\varepsilon^2,
\end{equation}
and hence
\begin{equation}
H^2(P,Q)\ge (13-5)\varepsilon-12.5\,\varepsilon^2
=8\,\varepsilon-12.5\,\varepsilon^2>7.5\,\varepsilon
\quad\text{for all }0<\varepsilon<\tfrac{1}{25}.
\label{eq:Hellinger_lowerbound}
\end{equation}

\textbf{Implication for learners.}
Suppose a learner outputs a distribution $R$ such that for the unknown $U\in\{P,Q\}$,
\begin{equation}
D_{\mathrm{KL}}(U\|R)\le \varepsilon.
\end{equation}
Since $H^2(P,Q)\le D_{\mathrm{KL}}(P\|Q)$, we have $H(U,R)\le \sqrt{\varepsilon}$. 
Let $V$ be the other element of $\{P,Q\}$. 
By the triangle inequality for the Hellinger distance and~\eqref{eq:Hellinger_lowerbound},
\begin{equation}
H(R,V)\ge H(P,Q)-H(U,R)
>\sqrt{7.5\,\varepsilon}-\sqrt{\varepsilon}
=(\sqrt{7.5}-1)\sqrt{\varepsilon},
\end{equation}
and therefore
\begin{equation}
D_{\mathrm{KL}}(R\|V)\ge H^2(R,V)\ge (\sqrt{7.5}-1)^2\,\varepsilon>3\,\varepsilon>\varepsilon.
\end{equation}
Thus, a learner within $\varepsilon$-KL of one distribution must be at least a constant factor more than $\varepsilon$ away from the other.

\textbf{Mean-testing consequence.} 
Setting $x_0=0$ and $x_1=1$, the mean gap is
\begin{equation}
\Delta=|\mathbb{E}_P[X]-\mathbb{E}_Q[X]|=24\,\varepsilon,
\end{equation}
and the variances are bounded by constants. 
Assume that from $n$ samples of the true distribution, a learner can construct a simulator $R$ such that 
$D_{\mathrm{KL}}(U\|R)\le\varepsilon$ for the true $U\in\{P,Q\}$. 
Suppose we are given $n$ samples from one of $P$ or $Q$ (unknown to us), and we use them to produce such a simulator $R$. 
Since $R$ is within $\varepsilon$-KL of one distribution but more than $\varepsilon$ away from the other as proved above, 
we can generate arbitrarily many samples from $R$ and thereby determine whether the original samples came from $P$ or from $Q$. 
Thus, the ability to learn a simulator $R$ within $\varepsilon$-KL error implies the ability to distinguish $P$ from $Q$.

Although $P$ and $Q$ are separated by $D_{\mathrm{KL}}(P\|Q) > 20\varepsilon$, their means differ only by $\Theta(\varepsilon)$, i.e., their separation is linear in $\varepsilon$. 
If a simulator $R$ achieves $D_{\mathrm{KL}}(U\|R) \le \varepsilon$ for some $U \in \{P, Q\}$, it cannot simultaneously satisfy this bound for both distributions. 
Therefore, distinguishing whether the samples originated from $P$ or $Q$ means that we can distinguish two distributions whose means differ by  $\Theta(\varepsilon)$. 
By standard mean-testing lower bounds \citep{Wainwright_2019}, distinguishing two distributions whose means differ by $\Theta(\varepsilon)$ requires at least $\Omega(1/\varepsilon^2)$ samples. 
Hence, any learner that produces such an $R$ with $D_{\mathrm{KL}}(U\|R)\le\varepsilon$ must necessarily use $\Omega(1/\varepsilon^2)$ samples.

\qedhere

\section{Dimension of the discrete-state score} \label{app: score_dim}

In principle, if we view the discrete-state diffusion process as a continuous-time Markov chain (CTMC) on the full
state space $\mathcal{X} = [S]^d$ with $N = S^d$ states, then the forward rate matrix $Q_t$ would be of size
$N \times N$, and the corresponding discrete-state score
$s_t(\mathbf{x}) = \bigl(q_t(\mathbf{y})/q_t(\mathbf{x})\bigr)_{\mathbf{y} \neq \mathbf{x}}$
would indeed live in $\mathbb{R}^{N-1}$.
However, such a formulation is computationally intractable for high-dimensional data.

Following prior works \citep{chen2024convergenceanalysisdiscretediffusion} and \citep{zhang2025convergencescorebaseddiscretediffusion}, we impose a
\emph{coordinate-wise (token-wise) factorization} of the forward process.
Specifically, each of the $d$ coordinates evolves independently according to a shared $S \times S$ rate
matrix.
Under this assumption, every jump of the forward CTMC modifies exactly \emph{one coordinate at a time}, and
transitions are only allowed between states whose Hamming distance is one.
As a result, the full rate matrix $Q_t$ is extremely sparse: from any state
$\mathbf{x} \in [S]^d$, there are only $d(S-1)$ nonzero off-diagonal transitions.

The reverse-time CTMC inherits exactly the same sparsity structure.
Its rate matrix $Q_t^{\leftarrow}$ involves only transitions of the form
\[
\mathbf{x} \;\longrightarrow\; \mathbf{x}_{\setminus i} \odot \hat{x}_i,
\qquad i \in [d],\ \hat{x}_i \neq x_i,
\]
and for each such transition the reverse rate is given by
\[
Q_t^{\leftarrow}\!\left(\mathbf{x}, \mathbf{x}_{\setminus i} \odot \hat{x}_i\right)
\;=\;
Q^{\mathrm{tok}}_{T-t}(\hat{x}_i, x_i)\,
\frac{q_t(\mathbf{x}_{\setminus i} \odot \hat{x}_i)}{q_t(\mathbf{x})}.
\]
Thus, the only probability ratios that appear in the reverse dynamics are those corresponding to
single-coordinate replacements.

This observation leads us to express the discrete-state score in the form required by the reverse-time CTMC.
We therefore define
\[
s_t(\mathbf{x})_{i,\hat{x}_i}
\;:=\;
\frac{q_t(\mathbf{x}_{\setminus i} \odot \hat{x}_i)}{q_t(\mathbf{x})},
\qquad i \in [d],\ \hat{x}_i \neq x_i,
\]
which parametrizes all nonzero transition rates of $Q_t^{\leftarrow}$.
For each coordinate $i$, the collection
$\{ s_t(\mathbf{x})_{i,\hat{x}_i} : \hat{x}_i \neq x_i \}$
has $S-1$ degrees of freedom, since the probabilities over the $S$ possible symbols at that coordinate must
sum to one after normalization.
Aggregating across all $d$ coordinates yields a score representation in
$\mathbb{R}^{d(S-1)}$.

Importantly, this reduction from $\mathbb{R}^{S^d-1}$ to $\mathbb{R}^{d(S-1)}$ is not an approximation or
modeling choice.
It is a consequence of the sparsity of the forward and reverse generators under the factorized CTMC
assumption, and it provides a complete and tractable characterization of the reverse diffusion dynamics.

\section{PL Condition and Sample-complexity Scaling}
\label{app:pl_sample_complexity}

This section clarifies the role of the Polyak--{\L}ojasiewicz (PL) condition in our analysis and explains how the stated sample-complexity bounds should be interpreted. 

\subsection{Role and Interpretation of the PL Condition}
\label{app:pl_interpretation}

This section clarifies the role of the Polyak--{\L}ojasiewicz (PL) condition in the analysis and explains how it should be interpreted within the theoretical framework of our work.

The PL condition is imposed as a \emph{structural regularity assumption on the population loss landscape} $\mathcal{L}_k(\theta)$ at each discrete time index $k$. Its purpose is to enable control of the optimization error by relating the gradient norm to function-value suboptimality. Importantly, this assumption is not tied to strong convexity, nor is it intended to characterize the degree of overparameterization of the neural network used to parameterize the score.

While overparameterized neural networks form one prominent class of models where PL-type or gradient-dominance conditions are known to hold, such conditions have been established in substantially broader settings. In particular, PL or closely related gradient-dominance properties are known to hold even in regimes that are not heavily overparameterized, \emph{i.e.}, when the number of parameters doesn't exceed the number of effective constraints.

Prior work by \citet{zhou2017characterizationgradientdominanceregularity} rigorously characterizes gradient-dominance and regularity conditions for several nonconvex architectures, including deep linear networks, deep linear residual networks, and one-hidden-layer nonlinear networks, in neighborhoods of global minimizers, without requiring a heavily overparameterized regime. These results show that even relatively modestly sized networks can exhibit a PL-type landscape, at least locally around solutions. This type of geometric behavior is precisely what is required in our analysis: a condition ensuring that small gradient norms imply small suboptimality, which in turn yields linear-type convergence rates for stochastic gradient methods.

In our framework, the PL condition is imposed on each time-indexed \emph{population loss} $\mathcal{L}_k(\theta)$, rather than on the empirical objective. It is used solely as an analytical device to convert bounds on gradient norms into bounds on function-value suboptimality, thereby enabling control of both statistical and optimization errors. The sample-complexity requirement appearing in the main results should therefore be interpreted as a \emph{sufficient condition} ensuring that these errors are controlled at the desired level. This requirement does not preclude the population loss from satisfying a PL condition in regimes where the number of samples is comparable to, or larger than, the number of trainable parameters.

\subsection{Sample-Complexity Scaling and Width Assumptions}
\label{app:sample_complexity_width}

This section clarifies the interpretation of the sample-complexity bounds derived in the paper and explains the role of the width condition appearing in Lemma \ref{lemma:approx_error}. 

The main analysis yields a sample-complexity bound of the form
\begin{equation}
n \gtrsim \Omega\!\left(W^{2L}\right),
\end{equation}
where $W$ denotes the width of the neural network parameterizing the score function and $L$ denotes its depth. This dependence arises from standard Rademacher complexity arguments for deep neural networks, which are known to be conservative, particularly with respect to depth. As a result, the bound should be interpreted as a worst-case sufficient condition ensuring convergence, rather than as a sharp or practically tuned estimate of the number of training samples required. Exponential or highly pessimistic depth dependence is common in current deep-learning theory such as in \citet{benton2024errorboundsflowmatching} and often substantially overestimates actual data requirements. 

When this bound is combined with the sufficient-width condition $W \ge (S-1)d$ used in Lemma~1, the resulting scaling becomes $\Omega((Sd)^{2L})$, which can be numerically large for large vocabularies $S$ or deep networks $L$. This behavior reflects the proof technique and worst-case generalization guarantees rather than an inherent limitation of discrete-state diffusion models. Importantly, the substitution $W = (S-1)d$ is not required by the main generalization theorem itself; it is introduced solely as a convenient sufficient condition in Lemma \ref{lemma:approx_error} to ensure zero approximation error via an explicit construction.

More generally, when $W < (S-1)d$ such as with parameter sharing or low-rank structure, the analysis extends naturally by introducing an explicit approximation error term. Let
\begin{equation}
\varepsilon_{\mathrm{approx}}(W,S,d)
\end{equation}
denote the infimum over all width-$W$ networks in the considered class of the population approximation error to the target score function. Under this formulation, Lemma \ref{lemma:approx_error} and the main excess-risk bound can be restated with an additional additive term $\varepsilon_{\mathrm{approx}}(W,S,d)$. In the special case $W \ge (S-1)d$, the explicit construction implies $\varepsilon_{\mathrm{approx}}(W,S,d)=0$, recovering the original statement of Lemma \ref{lemma:approx_error}.

This perspective is standard in learning theory. Approximation error induced by restricting to a given function class is often treated as a fixed constant once the class is chosen, allowing the analysis to focus on estimation and optimization errors that depend on the sample size. For example, in PAC-Bayesian analyses, approximation or misspecification errors are absorbed into a constant term once the hypothesis class is fixed \citep{mai2025pacbayesianriskboundsfully}. Similarly, in NTK or RKHS based analyses of neural networks, it is typically assumed that the target function lies in, or is well approximated by, the chosen function class, with the resulting misspecification error represented as a constant additive term \citep{bing2025kernelridgeregressionpredicted}. Existing theoretical analyses of diffusion models adopt analogous assumptions, treating approximation error as fixed \citep{guptaimproved,guan2025mirrorflowmatchingheavytailed, gaur2025generativemodelingcontinuousflows}.

\section{Relationship to Prior Work}
\label{app:novelty}

This work is inspired by existing theoretical analyses of discrete-state diffusion models \citep{zhang2025convergencescorebaseddiscretediffusion,chen2024convergenceanalysisdiscretediffusion}, particularly those that decompose the overall error into interpretable components and study convergence under idealized assumptions. Similar decomposition is used in works in reinforcement learning such as \citep{pmlr-v202-gaur23a},  \citep{gaur2023global},\citep{gaur2024closinggapachievingglobal} as well as flow-matching \citep{gaur2025generativemodelingcontinuousflows} and diffusion models \citep{gaur2025improvedsamplecomplexitydiffusion}. We adopt a similar high-level decomposition as a conceptual starting point, but develop technical tools and results tailored to \emph{discrete-state diffusion models trained with finite-width neural networks under nonconvex objectives}.

At a high level, the novelty of this work lies in making the abstract components of prior discrete-state diffusion analyses \emph{explicit, data-dependent, and architecture-aware} in the discrete-state setting.

\paragraph{Approximation error (Lemma \ref{lemma:approx_error}).}
Prior analyses typically assume realizability of the score function or work in a fixed functional class, thereby avoiding explicit treatment of approximation error. In contrast, Lemma \ref{lemma:approx_error} provides a discrete-state-specific approximation result that characterizes when a finite-width neural network can approximate the discrete-state score function and how this depends on the vocabulary size $S$ and dimension $d$. This lemma explicitly connects neural network expressivity to the statistical behavior of discrete-state diffusion models, a connection that does not appear in existing analyses.

\paragraph{Statistical error (Lemma \ref{lemma:stat_error}).}
While previous works derive convergence bounds at an abstract functional level, our analysis yields an explicit, training data dependent sample-complexity bound under a nonconvex PL-type landscape. This result makes the residual error term in prior analyses concrete and interpretable in terms of network width, depth, vocabulary size, and diffusion time. As a result, the statistical component of the error decomposition becomes an explicit quantity rather than an implicit asymptotic remainder.

\paragraph{Optimization error (Lemma \ref{lemma:opt_error}).}
We further analyze how optimization error propagates through the diffusion time horizon. This discrete-time stability analysis quantifies how local errors accumulate across steps and affect the final generated distribution. Such propagation effects are absent from prior analyses as they assume access to an empirical risk minimizer (ERM) and are essential for understanding convergence in the discrete-state setting with finite optimization steps.


Together, these results extend existing theory to a substantially more realistic setting, yielding guarantees that are not accessible through existing analyses.

\end{document}